\documentclass[12pt]{article}
\usepackage[margin=1.25in]{geometry}
\usepackage{times}
\usepackage{amsmath, amsfonts, bm}
\usepackage{amsthm,graphicx}
\usepackage[dvipsnames]{xcolor}
\usepackage{mathtools}
\usepackage{amssymb}
\usepackage{pifont}

\usepackage{makecell}
\usepackage[utf8]{inputenc} 
\usepackage[T1]{fontenc} 
\usepackage{amsfonts}  
\usepackage{mathrsfs} 
\usepackage{xspace}

\allowdisplaybreaks

\usepackage[unicode=true,
 bookmarks=false,
 breaklinks=false,pdfborder={0 0 1},colorlinks=false]
 {hyperref}
\hypersetup{
 colorlinks,citecolor=blue,filecolor=blue,linkcolor=blue,urlcolor=blue}

\usepackage{algorithm}
\usepackage{array}
\usepackage{makecell}
\usepackage{amssymb}
\usepackage{multirow}
\usepackage{color}
\usepackage[english]{babel}
\usepackage{graphicx}
\usepackage{natbib}
\usepackage{wrapfig}
\usepackage{epstopdf}
\usepackage{url}
\usepackage{graphicx}
\usepackage{color}
\usepackage{epstopdf}
\usepackage{algpseudocode}

\usepackage[T1]{fontenc}
\usepackage{bbm}
\usepackage{comment}
\usepackage{tikz}
\usepackage{tikzit}

\tikzstyle{Default}=[fill=white, draw=black, shape=circle]
\tikzstyle{Rec}=[fill=white, draw=black, shape=rectangle]

\tikzstyle{Unidirectional}=[->]
\tikzstyle{Double Arrow}=[<->]
\tikzstyle{Line}=[-]
\tikzstyle{Dashed arrow}=[->, dashed]
\tikzstyle{Dashed double arrow}=[<->, dashed]
\tikzstyle{dashed line}=[-, dashed]
\tikzstyle{blue line}=[-, color=blue]
\tikzstyle{Dotted Arrow}=[->, dotted]

\usepackage{booktabs,caption}
\usepackage[flushleft]{threeparttable}

\usepackage{tikz}
\usepackage{hyperref}
\hypersetup{colorlinks=true,citecolor=blue,linkcolor=red}

\usetikzlibrary{arrows}

\usepackage{amsmath}
\usepackage{amsthm}
\usepackage{subfig}
\usepackage{enumitem}

\newtheorem{theorem}{Theorem}
\newtheorem{lemma}{Lemma}
\newtheorem{definition}{Definition}
\newtheorem{proposition}{Proposition}
\newtheorem{corollary}{Corollary}

\newtheorem{remark}{Remark}

\newtheorem{assumption}{Assumption}

\usepackage{mdframed}
\mdtheorem{problem}{Problem}

\newcommand{\eps}{\epsilon}
\newcommand{\B}{\mathcal{B}}
\newcommand{\A}{\mathcal{A}}
\newcommand{\N}{\mathcal{N}}
\newcommand{\R}{\mathbb{R}}

\newcommand{\E}{\mathbb{E}}

\newcommand{\OVT}{\overline{V}^t}
\newcommand{\D}{\mathcal{D}}
\newcommand{\FC}{\mathcal{F}}
\newcommand{\GC}{\mathcal{G}}
\newcommand{\LC}{\mathcal{L}}
\newcommand{\T}{\mathcal{T}}
\newcommand{\VC}{\mathcal{V}_{\rho}}
\newcommand{\Fil}{\mathfrak{F}}
\newcommand{\Var}{\text{Var}}
\newcommand{\VB}{V_{\max}}
\newcommand{\BO}{\mathcal{O}}
\newcommand{\TO}{\widetilde{\mathcal{O}}}
\newcommand{\ZC}{\mathcal{Z}_{\rho}}
\newcommand{\DBE}{\mathrm{dim}_{\mathrm{BEE}}}
\newcommand{\DDE}{\mathrm{dim}_{\mathrm{DE}}}
\newcommand{\DE}{\mathrm{dim}_{\mathrm{E}}}
\newcommand{\Q}{\mathcal{Q}}
\newcommand{\OVTR}{\overline{V}^t_r}
\newcommand{\OVTG}{\overline{V}^t_g}
\newcommand{\FCR}{\mathcal{F}^r}
\newcommand{\GCR}{\mathcal{G}^r}
\newcommand{\FCG}{\mathcal{F}^g}
\newcommand{\GCG}{\mathcal{G}^g}
\newcommand{\VR}{V_{r,1}}
\newcommand{\VG}{V_{g,1}}
\newcommand{\VRH}{V_{r,h}}
\newcommand{\VGH}{V_{g,h}}
\newcommand{\QRH}{Q_{r,h}}
\newcommand{\QGH}{Q_{g,h}}
\newcommand{\QR}{Q_{r}}
\newcommand{\QG}{Q_{g}}
\newcommand{\OP}{\mu^*_{\text{CMDP}}}
\newcommand{\ebe}{\eps_{\text{BEE}}}
\newcommand{\SC}{\mathcal{S}}
\newcommand{\HC}{\mathcal{H}}
\newcommand{\reg}{\text{Regret}}
\newcommand{\vio}{\text{Violation}}
\newcommand{\X}{\mathcal{X}}
\newcommand{\WC}{\mathcal{W}}
\newcommand{\MG}{\mathcal{M}_{\text{MG}}}
\newcommand{\MC}{\mathcal{M}_{\text{CMDP}}}
\newcommand{\MV}{\mathcal{M}_{\text{VMDP}}}

\newcommand{\ba}{\boldsymbol{a}}
\newcommand{\muc}{\mu^*_{\text{CMDP}}}

\newcommand{\dbee}{d_{\mathrm{BEE}}}
\newcommand{\dbeei}{d_{\mathrm{BEE},i}}
\newcommand{\ncov}{\mathcal{N}_{\mathrm{cov}}}
\newcommand{\ncovi}{\mathcal{N}_{\mathrm{cov},i}}
\newcommand{\ncovr}{\mathcal{N}_{\mathrm{cov},r}}
\newcommand{\ncovg}{\mathcal{N}_{\mathrm{cov},g}}
\newcommand{\dbeer}{d_{\mathrm{BEE},r}}
\newcommand{\dbeeg}{d_{\mathrm{BEE},g}}
\newcommand{\I}{\mathcal{I}}

\newcommand{\sla}{\lambda_{\mathrm{sla}}}
\newcommand{\LO}{\mathcal{L}_{\mathrm{CMDP}}}
\newcommand{\tmu}{\widetilde{\mu}}
\newcommand{\dpi}{\mathcal{P}}
\newcommand{\cpi}{\mathrm{conv}(\mathcal{P}_{\Pi})}
\newcommand{\cpc}{\mathrm{conv}(\Pi)}
\newcommand{\ttau}{\widetilde{\tau}}
\newcommand{\br}{\boldsymbol{r}}
\newcommand{\BL}{\mathbb{B}}
\newcommand{\bv}{\boldsymbol{V}}
\newcommand{\bq}{\boldsymbol{Q}}
\newcommand{\CS}{\mathcal{C}}
\newcommand{\dist}{\mathrm{dist}}
\newcommand{\bx}{\boldsymbol{x}}
\newcommand{\LV}{\mathcal{L}_{\mathrm{VMDP}}}
\newcommand{\bth}{\boldsymbol{\theta}}

\newcommand{\PET}{\underline{\boldsymbol{V}}^t}
\newcommand{\projb}{\text{Proj}_{\mathbb{B}(1)}}
\newcommand{\dbeev}{d_{\mathrm{BEE,V}}}
\newcommand{\ncovv}{\mathcal{N}_{\mathrm{cov,V}}}
\newcommand{\PETJ}{\underline{{V}}^{t,j}}
\newcommand{\muv}{\mu^*_{\mathrm{VMDP}}}
\newcommand{\pmu}{P(\bv^{\mu^*_{\mathrm{VMDP}}}_1(s_1))}
\newcommand{\deff}{d_{\mathrm{eff}}}

\newcommand{\bmu}{\overline{\mu}}

\newcommand{\bpi}{\overline{\pi}}

\newcommand{\mainalg}{\texttt{DORIS}\xspace}
\newcommand{\cmdpalg}{\texttt{DORIS-C}\xspace}
\newcommand{\pealg}{\texttt{OptLSPE}\xspace}
\newcommand{\spealg}{\texttt{OptLSPE-C}\xspace}
\newcommand{\vmdpalg}{\texttt{DORIS-V}\xspace}
\newcommand{\vpealg}{\texttt{OptLSPE-V}\xspace}

\newcommand{\JC}{\mathcal{J}}
\DeclareRobustCommand{\rchi}{{\mathpalette\irchi\relax}}
\newcommand{\irchi}[2]{\raisebox{\depth}{$#1\chi$}}
\newcommand{\proj}{\text{Proj}_{[0,\rchi]}}

\definecolor{yxc}{RGB}{255,0,0}
\definecolor{yjc}{RGB}{190,0,255}
\definecolor{whz}{RGB}{1,11,111}
\definecolor{hbh}{RGB}{1,150,20}

\definecolor{zhuoran}{RGB}{0, 102, 204}

\ifdefined\usebigfont

\usepackage{times}
\usepackage[fontsize=13pt]{scrextend}
\AtBeginDocument{\newgeometry{letterpaper,left=1.56in,right=1.56in,top=1.71in,bottom=1.77in}}
\else
\fi

\begin{document}

\title{Decentralized Optimistic Hyperpolicy Mirror Descent: Provably No-Regret Learning in Markov Games
}
\author{%
	Wenhao Zhan\thanks{Princeton University. Email: \texttt{wenhao.zhan@princeton.edu}} \\
	\and
	Jason D. Lee\thanks{Princeton University. Email: \texttt{jasonlee@princeton.edu}} 
    \and
    Zhuoran Yang\thanks{Yale University. Email: \texttt{zhuoran.yang@yale.edu}}\\
}

\date{\today}
\maketitle

\begin{abstract}
We study decentralized policy learning in Markov games where we control a single agent to play with nonstationary and possibly adversarial opponents. Our goal is to develop a no-regret online learning algorithm that (i) takes actions based on the local information observed by the agent and (ii) is able to find the best policy in hindsight. For such a problem, the nonstationary state transitions due to the varying opponent pose a significant challenge. In light of a recent hardness result \citep{liu2022learning}, we focus on the setting where the opponent's previous policies are revealed to the agent for decision making. With such an information structure, we propose a new algorithm, \underline{D}ecentralized \underline{O}ptimistic hype\underline{R}policy m\underline{I}rror de\underline{S}cent (DORIS), which achieves $\sqrt{K}$-regret in the context of general function approximation, where $K$ is the number of episodes. Moreover, when all the agents adopt DORIS, we prove that their mixture policy constitutes an approximate coarse correlated equilibrium. In particular, DORIS maintains a \textit{hyperpolicy} which is a distribution over the policy space. The hyperpolicy is updated via mirror descent, where the update direction is obtained by an optimistic variant of least-squares policy evaluation. Furthermore, to illustrate the power of our method, we apply DORIS to constrained and vector-valued MDPs, which can be formulated as zero-sum Markov games with a fictitious opponent. 

\end{abstract}
\section{Introduction} \label{sec:intro}

Multi-agent reinforcement learning (MARL) studies how each agent  learns to maximize its  cumulative rewards 
by interacting with the environment as well as other agents, where the state transitions and rewards are affected by the actions of all the agents. 
Equipped with  powerful function approximators  such as  deep neural networks \citep{lecun2015deep}, MARL has achieved significant empirical success in various domains including the game of Go \citep{silver2016mastering}, StarCraft \citep{vinyals2019grandmaster},  DOTA2 \citep{berner2019dota}, Atari \citep{mnih2013playing}, multi-agent robotics systems \citep{brambilla2013swarm} and autonomous driving\citep{shalev2016safe}. 
Compared with the centralized setting  where a  central controller collects the information of all agents and coordinates their behaviors, decentralized algorithms \citep{gupta2017cooperative,rashid2018qmix} where each agent autonomously chooses its action based on its own local information  are often more desirable in MARL applications.
In specific,
decentralized methods 
  (1) are easier to implement and enjoy better scalability,  (2) are more robust to possible adversaries,  and (3) require less communication overhead \citep{hernandez2018multiagent,hernandez2019survey,canese2021multi,zhang2021multi,gronauer2022multi}.

In this work, we aim to design a provably efficient decentralized reinforcement learning (RL) algorithm in the online setting with function approximation. 
In the sequel, for the ease of presentation, 
we refer to  the controllable agent as 
the \textit{player}
and regard the rest of the agents as a meta-agent, called the  \textit{opponent}, which specifies its policies arbitrarily.  Our goal is to maximize the cumulative rewards of the player in the face of a possibly adversarial opponent, in the online setting where the policies of the player and opponent can be based on adaptively gathered local information.

From a theoretical perspective, arguably the most distinctive challenge of 
the 
decentralized setting 
is \textit{nonstationarity}.
That is, from the perspective of any agent, the states transitions  are affected by the policies of other agents in an unpredictable and potentially adversarial way and are thus nonstationary. 
This is in stark contrast to the centralized setting which can be regarded as a standard RL problem for the central controller which decides the actions for all the players. 
Furthermore, 
in the online setting, as the environment is unknown, to achieve sample efficiency, the player needs to strike a balance between \textit{exploration} and \textit{exploitation} in the  
context of function approximation and in the presence of an adversarial opponent. 
The dual challenges of nonstationarity and efficient exploration  are thus intertwined, making it challenging to develop provably efficient  decentralized MARL algorithms.


Consequently, there seem only limited theoretical understanding of the decentralized MARL setting with a  possibly adversarial opponent. Most of the existing algorithms \citep{brafman2002r,wei2017online,tian2021online,jin2021power,huang2021towards} can only compete against the Nash value of the  Markov game when faced with an arbitrary opponent. 
This is  a much weaker baseline compared with the results in classic matrix games \citep{fudenberg1991game,anderson2008theory} where the player is required to compete against \textit{the best fixed policy in hindsight}. Meanwhile, \cite{liu2022learning} seems  the only work we know that can achieve no-regret learning in MARL against the best hindsight policy, which focuses on the   \textit{policy revealing} setting 
where the player observes the policies played by the opponent in previous episodes. 
Moreover, the algorithm and theory in this work are limited to tabular cases and fail to deal with large or even continuous state and action space. 
To this end, we  would like to answer the following question: 

\begin{center}
\textbf{\textit{Can we design a decentralized MARL algorithm that provably achieves no-regret  against the best fixed policy in hindsight in the context of function approximation?}}
\end{center}

In this work, we provide a positive answer to the above question under the \textit{policy revealing} setting with general function approximation. 
In specific, we propose an actor-critic-type algorithm  \citep{konda1999actor} called \mainalg, which maintains a distribution over the policy space, named \textit{hyperpolicy}, for decision-making.  
To combat  the nonstationarity, \mainalg updates the hyperpolicy via  mirror descent (or equivalently, Hedge \citep{freund1997decision}).
Furthermore, to encourage exploration, the descent directions of  mirror descent   are obtained by solving optimistic variants of policy evaluation subproblems with general function approximation, which  only involve the local information of the player.
Under standard regularity assumptions on the underlying function classes, we prove that \mainalg achieves a sublinear regret in the presence of an adversarial opponent. 
In addition, when the agents all adopt \mainalg independently,  we prove that their average policy constitutes an approximate coarse correlated equilibrium. 
At the core of our analysis is a new complexity measure of function classes   that is tailored to the decentralized MARL setting. 
Furthermore, to demonstrate the power of \mainalg, we adapt it for solving constrained Markov decision process (CMDP) and vector-valued Markov decision process (VMDP), which can both be formulated as a zero-sum Markov game with a fictitious opponent. 
 
\paragraph{Our Contributions.} Our contributions are four-fold. First, we propose a new decentralized policy optimization algorithm, \mainalg,  that provably achieves no-regret in the context of general function approximation. 
As a result, when all agents adopt \mainalg, their average policy converges to a CCE of the Markov game. 
Secondly, we propose a new complexity measure named Bellman Evaluation Eluder dimension, which generalizes Bellman Eluder dimension \citep{jin2021bellman} for single-agent MDP to decentralized learning in Markov games, which might be of independent interest. 
Third, we modify \mainalg for solving CMDP with general function approximation, which is shown to achieve sublinear regret and constraint violation. 
Finally, we extend  \mainalg  to solving the approchability task \citep{miryoosefi2019reinforcement} in vector-valued Markov decision process (VMDP) and attain a near-optimal solution. To our best knowledge, \mainalg seems the first provably efficient decentralized  algorithm for achieving no-regret in MARL with general function approximation.

\
\subsection{Notations} 
In this paper we let $[n]=\{1,\cdots,n\}$ for any integer $n$. We denote the set of probability distributions over any set $\SC$ by $\Delta_{\SC}$ or $\Delta(\SC)$. We also let $\Vert\cdot\Vert$ denote the $\ell_2$-norm by default.

\subsection{Related works}
\paragraph{Decentralized learning with an adversarial opponent.} There have been a few works studying decentralized policy learning in the presence of a possibly adversarial opponent. \cite{brafman2002r} proposes R-max and is able to attain an average game value close to the Nash value in tabular MGs. 
More recently,   \cite{wei2017online,tian2021online} improve the regret bounds in tabular cases and \cite{jin2021power,huang2021towards} extend the results to general function approximation setting. 
However, these    works only compete against the Nash value of the game and are  unable to exploit the opponent. 
A more related paper is \cite{liu2022learning}, which develops a provably efficient algorithm that achieves a sublinear regret against the best fixed policy in hindsight. 
But there results are only limited to the tabular case. 
Our work extends the results \cite{liu2022learning} to the setting with  general function approximation, which requires novel technical analysis. 


\paragraph{Finding equilibria in self-play Markov games.} 
Our  work is closely related to the recent literature on  finding equilibria in  Markov games  via reinforcement learning. 
Most of the existing works  focus on two-player zero-sum games and consider centralized algorithms with unknown model dynamics. For example, \cite{wei2017online, bai2020provable} utilize   optimism to tackle the exploration-expoitation tradeoff and find Nash equilibria in tabular cases, and \cite{xie2020learning,jin2021power,huang2021towards} extend the  results to linear and general function approximation setting. 
Furthermore, 
under the decentralized setting with well-explored data, 
  \cite{daskalakis2020independent,zhang2021gradient,sayin2021decentralized,wei2021last,leonardos2021global,ding2022independent} utilize independent policy gradient algorithms to deal with potential Markov games and two-player zero-sum games.
  Meanwhile, under the online setting, \cite{bai2020near,mao2021on,jin2021v} design algorithms named V-learning,  which are able to find CCE in  multi-agent general-sum games. However, there  results are only limited to the   tabular case. 

\paragraph{Constrained Markov decision process.} \cite{efroni2020exploration,ding2021provably} propose a series of primal-dual algorithms for CMDPs which achieve $\sqrt{K}$ bound on regrets and constraint violations in tabular and linear approximation cases. \cite{liu2021learning} reduces the constraint violation to $\TO(1)$ by adding slackness to the algorithm and achieves zero violation when a strictly safe policy is known; \cite{wei2021provably} further avoids such requirement with the price of worsened regrets. Nevertheless, these improvements are only discussed in the tabular case.

\paragraph{Approchability for vector-valued Markov decision process.} \cite{miryoosefi2019reinforcement} first introduces the approachability task for VMDPs but does not provide an algorithm with polynomial sample complexity. Then \cite{yu2021provably} proposes a couple of primal-dual algorithms to solve this task and achieves a $\TO(\epsilon^{-2})$ sample complexity in the  tabular case. More recently, \cite{miryoosefi2021simple} utilizes reward-free reinforcement learning to tackle the problem and studies both the tabular and linear approximation cases, achieving roughly the same sample complexity as \cite{yu2021provably}.

\section{Preliminaries}
\label{sec:setting}


\subsection{General-Sum Markov Games}
Let us consider an $n$-agent general-sum Markov game (MG) $\MG=(\mathcal{S},\{\A_i\}_{i=1}^{n},\{P_h\}_{h=1}^H,\break\{r_{h,i}\}_{h=1,i=1}^{H,n},H)$, where $\mathcal{S}$ is the state space, $\A_i$ is the action space of $i$-th agent, $P_h: \mathcal{S}\times\prod_{i=1}^n\mathcal{A}_i\to\Delta(\mathcal{S})$ is the transition function at $h$-th step, $r_{h,i}:\mathcal{S}\times\prod_{i=1}^n\mathcal{A}_i\to\R_+$ is the reward function of $i$-th agent at $h$-th step and $H$ is the length of each episode. 

We assume each episode starts at a fixed start state $s_1$ and terminates at $s_{H+1}$. At step $h\in[H]$, each agent $i$ observes the state $s_h$ and takes action $a_{h,i}$ simultaneously. After that, agent $i$ receives its own reward $r_{h,i}(s_h,\ba_h)$ where $\ba_h:=(a_{h,1},\cdots,a_{h,n})$ is the joint action and the environment transits to a new state $s_{h+1}\sim P_h (\cdot|s_h,\ba_h)$.


\paragraph{Policy.}
A policy of the $i$-th agent $\mu_i=\{\mu_{h,i}:\SC\to\Delta_{\A_i}\}_{h\in[H]}$ specifies the action selection probability of agent $i$ in each state at each step. In the following discussion we will drop the $h$ in $\mu_{h,i}$ when it is clear from the context. We use $\pi$ to represent the joint policy of all agents and $\mu_{-i}$ to denote the joint policy of all agents other than $i$. Further, we assume each agent $i$ chooses its policy from a policy class $\Pi_i$. Similarly, let $\Pi_{-i}:=\prod_{j\neq i}\Pi_{j}$ denote the product of all agents' policy classes other than the $i$-th agent.

\paragraph{Value functions and Bellman operators.}
Given any joint policy $\pi$, the $i$-th agent's value function $V^{\pi}_{h,i}:\SC\to\R$ and action-value (or Q) function $Q^{\pi}_{h,i}:\SC\times\prod_{i=1}^n\mathcal{A}_i\to\R$ characterize its expected cumulative rewards given a state or a state-action pair, as defined below:
\begin{align*}
	&V^{\pi}_{h,i}(s):=\E_{\pi}\bigg[\sum_{t=h}^H r_{t,i}(s_t,\ba_{t})\bigg|s_h=s\bigg], Q^{\pi}_{h,i}(s,\ba):=\E_{\pi} \bigg [\sum_{t=h}^H r_{t,i} (s_t,\ba_{t}) \bigg| s_h=s,\ba_h=\ba \bigg ],
\end{align*}
where the expectation is w.r.t. to the distribution of the trajectory induced by executing the joint policy $\pi$ in $\MG$.
Here we suppose the action-value function is bounded:
\begin{align*}
Q^{\pi}_{h,i}(s,\ba)\leq\VB,\forall s,\ba,h,i,\pi.
\end{align*}
Notice that when the reward function is bounded in $[0,1]$, $\VB=H$ naturally.

\subsection{Decentralized Policy Learning}
In this paper we consider the decentralized learning setting \citep{jin2021power,huang2021towards,liu2022learning} where only one agent is under our control, which we call \textit{player}, and the other agents can be adversarial. Without loss of generality, assume that we can only control agent 1 and view the other agents as a meta \textit{opponent}. To simplify writing, we use $a_h,\A,r_h,\mu,\Pi, V^{\pi}_h, Q^{\pi}_h$ to denote $a_{h,1},\A_1,r_{h,1},\mu_1,\Pi_1,V^{\pi}_{h,1},Q^{\pi}_{h,1}$ respectively. We also use $b_h,\B,\nu,\Pi'$ to represent the joint action, the joint action space, the joint policy and the joint policy class of all the agents other than agent 1.

By \textit{decentralized learning} we mean that during the episode, the player can only observe its own rewards, actions and some information of the opponent specified by the protocol, i.e., $\{s_h^t,a_h^t,
\JC_h^{t}, r_{h}^t \}_{h=1}^H$ where $\{\JC_h\}_{h=1}^H$ is the information revealed by the opponent in each episode, which we will specify later. Then at the beginning of $t$-th episode, the player chooses a policy $\mu^t$ from its policy class $\Pi$ based only on its local information collected from previous episodes, without any coordination from a centralized controller. Meanwhile, the opponent selects $\nu^t$ from $\Pi'$ secretly and probably adversely.

The learning objective is to minimize the regret of the player by comparing its performance against the best fixed policy in hindsight as standard in online learning literature \citep{anderson2008theory,hazan2016introduction}:
\begin{definition}[Regret]
\label{def:regret}	
Suppose $(\mu^t,\nu^t)$ are the policies played by the player and the opponent in $t$-th episode. Then the regret for $K$ episodes is defined as
\begin{align}
\label{eq:reg def}
\reg(K)=\max_{\mu\in\Pi}\sum_{t=1}^KV^{\mu\times\nu^t}_1(s_1)-\sum_{t=1}^KV^{\mu^t\times\nu^t}_1(s_1),
\end{align}
where $\mu\times\nu$ denotes the joint policy where the player and the opponent play $\mu$ and $\nu$ independently.
\end{definition} 
Achieving low regrets defined in (\ref{eq:reg def}) indicates that, in the presence of an opponent playing adaptive  $\{ \nu^t \}_{t=1}^T$,  the algorithm approximately is as good as  the best fixed policy in $\Pi$ in the hindsight.



\paragraph{Relation between Definition~\ref{def:regret} and equilibria.} An inspiration for our definition of regrets comes from the tight connection between low regrets and equilibria in the matrix game \citep{fudenberg1991game,blum2007learning,daskalakis2011near}. By viewing each policy in the policy class as a pure strategy in the matrix game, we can generalize the notion of equilibria in matrix games to Markov games naturally. In particular, a correlated mixed strategy profile $\bpi$ can be defined as a mixture of the joint policy of all agents, i.e., $\bpi\in\Delta(\prod_{i\in[n]}\Pi_i)$.
Suppose the marginal distribution of $\bpi$ over the policy of agent $i$ is $\bmu_i$, then we can see that $\bmu_i$ is a mixture of the policies in $\Pi_i$. For a correlated profile, the agents might not play their mixed policies $\bmu_i$ independently, which means that $\bpi$ might not be the product of $\bmu_i$. A coarse correlated equilibrium (CCE) is simply a correlated profile that all the agents have no incentive to deviate from by playing a different independent policy:
\begin{definition}[Coarse correlated equilibrium (CCE) for $n$-player MG]
	A correlated strategy profile $\bpi$ is an $\eps$-approximate coarse correlated equilibrium if we have for all $i\in[n]$
	\begin{align}
	\label{eq:def cce}
		V_{1,i}^{\bpi}(s_1)\geq\max_{\mu'\in\Pi_i}V_{1,i}^{\mu'\times\bmu_{-i}}(s_1)-\eps,
	\end{align}
	where $\bmu_{-i}$ is the marginal distribution of $\bpi$ over the joint policy of all agents other than $i$.
\end{definition}
\begin{remark}
 Our definition of correlated strategy profile and CCEs is slightly different from \cite{mao2021on}. This is because we are considering with policy classes while \cite{mao2021on} does not. In fact, our definition is more strict in the sense that a correlated profile satisfying our definition must also satisfy theirs. 
\end{remark}
Specially, if a CCE $\bpi$ satisfies $\bpi=\prod_{i\in[n]}\bmu_i$, then we call $\bpi$ a Nash Equilibrium (NE). We will show in Section~\ref{sec:selfplay} that if a decentralized algorithm can achieve low regrets under Definition~\ref{def:regret}, we will be able to find an approximate  CCE by running the algorithm independently for each agent and return the resulting mixture policy. 
\subsection{Function Approximation}
To deal with the potentially large or even infinite state and action space, we consider learning with general value function approximation in this paper \citep{jiang2017contextual,jin2021bellman}. We assume the player is given a function class $\FC=\FC_1\times\cdots\times\FC_H$ ($\FC_h\subseteq(\SC\times\A\times\B\to[0,\VB])$) to approximate the action-value functions. Since there is no reward in state $s_{H+1}$, we let $f_{H+1}(s,a,b)=0$ for all $s\in\SC,a\in\A,b\in\B,f\in\FC$. 

To measure the size of $\FC$, we use $|\FC|$ to denote its cardinality when $\FC$ is finite. For infinite function classes, we use $\eps$-covering number to measure its size, which is defined as follows. 
\begin{definition}[$\eps$-covering number]
The $\eps$-covering number of $\FC$, denoted by $\N_{\FC}(\eps)$, is the minimum integer $n$ such that there exists a subset $\FC'\subset\FC$ with $|\FC'|=n$ and for any $f\in\FC$ there exists $f'\in\FC'$ such that $\max_{h\in[H]}\Vert f_h-f'_h\Vert_{\infty}\leq\eps$. 
\end{definition}

In addition to the size, we also need to impose some complexity assumption on the structure of the function class to achieve small generalization error. Here we introduce one of such structure complexity measures called Distributional Eluder (DE) dimension \citep{jin2021bellman}, which we will utilize in our subsequent analysis. First let us define independence between distributions as follows.

\begin{definition}[$\eps$-independence between distributions]
Let $\WC$ be a function class defined on $\X$, and $\rho,\rho_1,\cdots,\rho_n$ be probability measure over $\X$. We say $\rho$ is $\eps$-independent of $\{\rho_1,\cdots,\rho_n\}$ with respect to $\WC$ if there exists $w\in\WC$ such that $\sqrt{\sum_{i=1}^n(\E_{\rho_i}[w])^2}\leq\eps$ but $|\E_{\rho}[w]|>\eps$.
\end{definition}

From the definition we can see that a probability distribution $\rho$ is independent from $\{\rho_1,\cdots,\rho_n\}$ if there exists a discriminator function in $\WC$ such that the function values are small at $\{\rho_1,\cdots,\rho_n\}$ while large at $\rho$. Then DE dimension is simply the length of the longest sequence of independent probability distributions that the function class can discriminate. 
We lay out the definition of the DE dimension as follows. 

\begin{definition}[Distributional Eluder (DE) dimension]
Let $\WC$ be a function class defined on $\X$, and $\Q$ be a family of probability measures over $\X$. The distributional Eluder dimension $\DDE(\WC,\Q,\eps)$ is the length of the longest sequence $\{\rho_1,\cdots,\rho_n\}\subset\Q$ such that there exists $\eps'\geq\eps$ where $\rho_i$ is $\eps'$-independent of $\{\rho_1,\cdots,\rho_{i-1}\}$ for all $i\in[n]$.
\end{definition}

Eluder dimension, another commonly-used complexity measure proposed by \cite{russo2013eluder}, is a special case of DE dimension when the distributions concentrate on a single point. That is, if we choose $\Q=\{\delta_x(\cdot)|x\in\X\}$ where $\delta_x(\cdot)$ is the dirac measure centered at $x$, then the Eluder dimension can be formulated as
\begin{align*}
\DE(\WC,\eps)=\DDE(\WC-\WC,\Q,\eps),
\end{align*}
where $\WC-\WC=\{w_1-w_2: w_1,w_2\in\WC\}$. Many function classes in MDPs are known to have low Eluder dimension, including linear MDPs \citep{jin2020provably}, generalized linear complete models \citep{wang2019optimism} and kernel MDPs \citep{jin2021bellman}. 

We also assume the existence of an auxiliary function class $\GC=\GC_1\times\cdots\times\GC_H$ ($\GC_h\subseteq(\SC\times\A\times\B\to[0,\VB])$) to capture the results of applying Bellman operators on $\FC$ as in \cite{jin2021bellman,jin2021power}. When $\FC$ satisfies completeness (Assumption~\ref{ass:complete self}), we can simply choose $\GC=\FC$.

\section{Algorithm: \mainalg}
\label{sec:alg}


\paragraph{Policy revealing setting.} Recall that in decentralized policy learning setting, the player is also able to observe some information of the opponent, denoted by $\JC_h$, aside from its own actions and rewards. There have been works studying the case where $\JC_h=\emptyset$ \citep{tian2021online} and $\JC_h=b_h$ \citep{jin2021power,huang2021towards} in two-player zero-sum games. However, their benchmark is the Nash value of the Markov game, i.e., $V^{\mu^*\times\nu^*}_1(s_1)$ where $\mu^*\times\nu^*$ is an NE, which is strictly weaker than our benchmark $\max_{\mu\in\Pi}\sum_{t=1}^KV^{\mu\times\nu^t}_1(s_1)$ in two-player zero-sum games. In fact, \cite{liu2022learning} have showed achieving a  low regret under Definition~\ref{def:regret} is exponentially hard in tabular cases when the opponent's policy is not revealed. Therefore in this paper we let $\JC_h=\{b_h,\nu_h\}$ just like \cite{liu2022learning} and call this information structure \textit{policy revealing} setting.

That said, even in policy revealing setting, the challenge of nonstationarity still exists because the opponent's policy can be adversarial and only gets revealed after the player plays a policy. Thus from the perspective of the player, the transition kernel $P^{\nu}_h(\cdot|s,a):=\E_{b\sim\nu_h(s)}P_h(\cdot|s,a,b)$ still changes in an unpredictable way across episodes. In addition, the problem of how to balance exploration and exploitation with general function approximation also remains due to the unknown transition probability. In this section we propose \mainalg, an algorithm that is capable of handling all these challenges and achieving a $\sqrt{K}$ regret upper bound in the  policy revealing setting. 



\paragraph{\mainalg.} Intuitively, our algorithm is an actor-critic / mirror descent (Hedge) algorithm where each policy $\mu$ in $\Pi$ is regarded as an expert and the performance of each expert at episode $t$ is given by the value function of $V^{\nu\times\nu^t}_1(s_1)$. We call it \underline{D}ecentralized \underline{O}ptimistic hype\underline{R}policy m\underline{I}rror de\underline{S}cent (\mainalg). \mainalg possesses three important features, whose details are shown in Algorithm~\ref{alg:MDPS}:
\begin{itemize}
	\item \textbf{Hyperpolicy and Hedge:} Motivated from the adversarial bandit literature \citep{anderson2008theory,hazan2016introduction,lattimore2020bandit}, \mainalg maintains a distribution $p$ over the policies in $\Pi$, which we call \textit{hyperpolicy}, to combat the nonstaionarity. The hyperpolicy is updated using Hedge, with the reward of each policy $\mu$ being an estimation of the value function $V^{\mu\times\nu^t}_1(s_1)$. This is equivalent to running mirror ascent algorithm over the policy space $\Pi$ with the gradient being $V^{\mu\times\nu^t}_1(s_1)$.
	
	\item \textbf{Optimism:} However, we do not have access to the exact value function since the transition probability is unknown, which forces us to deal with the exploration-exploitation tradeoff. Here we utilize the \textit{Optimism in the Face of Uncertainty} principle \citep{azar2017minimax,jin2020provably,jin2021bellman,jin2021power,huang2021towards} and choose our estimation $\OVT(\mu)$ to be optimistic with respect to the true value $V^{\mu\times\nu^t}_1(s_1)$. In this way \mainalg will prefer policies with more uncertainty and thus encourage exploration in the Markov game. 
		
	\item \textbf{Optimistic policy evaluation with general function approximation:} Finally we need to design an efficient method to obtain such optimistic estimation $\OVT(\mu)$ with general function approximation. Here we propose \pealg to accomplish this task. In short, \pealg constructs a confidence set for the target action-value function $Q^{\mu\times\nu}$ based on the player's local information and chooses an optimistic estimation from the confidence set, as shown in Algorithm~\ref{alg:GOLF}. The construction of the confidence set utilizes the fact that $Q^{\mu\times\nu}_{h}$ satisfies the Bellman equation \citep{puterman1994markov}:
	\begin{align*}
	Q^{\mu\times\nu}_{h}(s,a,b)=(\T^{\mu,\nu}_h Q^{\mu\times\nu}_{h+1})(s,a,b):=r_h(s,a,b)+\E_{s'\sim P_h(\cdot|s,a,b)}[Q^{\mu\times\nu}_{h+1}(s',\mu,\nu)],
	\end{align*}
	where $Q^{\mu\times\nu}_{h+1}(s',\mu,\nu)=\E_{a'\sim\mu(\cdot|s'),b'\sim\nu(\cdot|s')}[Q^{\mu\times\nu}_{h+1}(s',a',b')]$. We call $\T^{\mu,\nu}_h$ the Bellman operator induced by $\mu\times\nu$ at $h$-th step. Then the construction rule of $\B_{\D}(\mu,\nu)$ is based on least-squared policy evaluation with slackness $\beta$ as below:
	\begin{align}
	\label{eq:eva}
		\B_{\D}(\mu,\nu)\gets \Big \{f\in\FC:\LC_{\D}(f_h,f_{h+1},\mu,\nu)\leq\inf_{g\in\GC}\LC_{\D}(g_h,f_{h+1},\mu,\nu)+\beta,\forall h\in[H] \Big \},
	\end{align}
	where $\LC_{\D}$ is the empirical Bellman residuals on $\D$:
	\begin{align*} \LC_{\D}(\xi_h,\zeta_{h+1},\mu,\nu)=\sum_{(s_h,a_h,b_h,r_h,s_{h+1})\in\D}[\xi_h(s_h,a_h,b_h)-r_h-\zeta_{h+1}(s_{h+1},\mu,\nu)]^2.
	\end{align*}
	
\end{itemize}


\begin{algorithm}[h]
	\caption{\textbf{\mainalg}}
	\label{alg:MDPS}
	\begin{algorithmic}
	\State \textbf{Input}: learning rate $\eta$, confidence parameter $\beta$.
	\State Initialize $p^1\in\Delta_{\Pi}$ to be uniform over $\Pi$.
	\For{$t=1,\cdots,K$}
	\State \textbf{Collect samples:}
	\State The player samples $\mu^t$ from $p^t$.
	\State Run $\pi^t=\mu^t\times\nu^t$ and collect $\D_{t}=\{s^t_1,a^t_1,b^t_1,r^t_1,\cdots,s^t_{H+1}\}$.
	\State \textbf{Update policy distribution:}
	\State The opponent reveals its policy $\nu^t$ to the player.
	\State $\OVT(\mu)\gets\textbf{\pealg}(\mu,\nu^t,\D_{1:t-1},\FC,\GC,\beta),\quad\forall\mu\in\Pi$.
	\State $p^{t+1}(\mu)\propto p^{t}(\mu)\cdot\exp(\eta\cdot\OVT(\mu)),\quad\forall\mu\in\Pi$.
	\EndFor
	\end{algorithmic}
\end{algorithm}


%

\begin{algorithm}[h]
	\caption{$\textbf{\pealg}(\mu,\nu,\D,\FC,\GC,\beta)$}
	\label{alg:GOLF}
	\begin{algorithmic}
		\State \textbf{Construct} $\B_{\D}(\mu,\nu)$ based on $\D$ via (\ref{eq:eva}).
        \State \textbf{Select} $\bar{V}\gets\max_{f\in\B_{\D}(\mu,\nu)} f(s_1,\mu,\nu)$.
		\State \Return $\bar{V}$.
	\end{algorithmic}
\end{algorithm}


\paragraph{Decentralized Algorithm.} Here we want to highlight that \mainalg is a decentralized algorithm because the player can run \mainalg based only on its local information, i.e., $\{s_h,a_h,\JC_h,r_h\}$, and we do not make any assumptions on the behavior of the opponent.

\subsection{\mainalg in Self-Play Setting}

Apart from decentralized learning setting with a possibly adversarial opponent, we are also interested in the self-play setting where we can control all the agents and need to find an equilibrium for the $n$-agent general-sum Markov game. Inspired by the existing relationships between no-regret learning and CCE in matrix games \citep{fudenberg1991game,blum2007learning,daskalakis2011near}, a natural idea is to simply let all agents run \mainalg independently. To achieve this, we assume each agent $i$ is given a value function class $\FC_i=\FC_{1,i}\times\cdots\times\FC_{H,i}$ and an auxiliary function class $\GC_i=\GC_{1,i}\times\cdots\times\GC_{H,i}$ as in \mainalg, and run \mainalg by viewing the other agents as its opponent. Suppose the policies played by agent $i$ during $K$ episodes are $\{\mu^t_{i}\}_{t=1}^{K}$, then we output the final joint policy as a uniform mixture of them:
\begin{align*}
	\widehat{\pi}\sim\mathrm{Unif}\Big(\Big \{ {\textstyle \prod_{i\in[n]}}~ \mu^1_{i},\cdots,{\textstyle \prod_{i\in[n]}}~ \mu^K_{i} \Big\} \Big ).
\end{align*}
See Algorithm~\ref{alg:MDPS self} for more details.
\begin{algorithm}[h]
	\caption{\textbf{\mainalg} in self-play setting}
	\label{alg:MDPS self}
	\begin{algorithmic}
		\State \textbf{Input}: learning rate $\{\eta_i\}_{i=1}^n$, confidence parameter $\{\beta_i\}_{i=1}^n$.
		\State Initialize $p^1_i\in\Delta_{|\Pi_i|}$ to be uniform over $\Pi_i$ for all $i\in[n]$.
		\For{$t=1,\cdots,K$}
		\State \textbf{Collect samples:}
		\State Agent $i$ samples $\mu^t_i$ from $p^t_i$.
		\State Run $\mu^t=\prod_{i=1}^n\mu^t_i$ and collect $\D_{t,i}=\{s^t_1,\ba^t_1,r^t_{1,i},\cdots,s^t_{H+1}\}$ for each agent $i$.
		\State \textbf{Update policy distribution:}
		\State All agents reveal their policies $\mu^t_i$.
		\State $\OVT_i(\mu_i)\gets\textbf{\pealg}(\mu_i,\mu^t_{-i},\D_{1:t-1,i},\FC_i,\GC_i,\beta_i),\quad\forall\mu_i\in\Pi_i,i\in[n]$.
		\State $p^{t+1}_i(\mu_i)\propto p^{t}_i(\mu_i)\cdot\exp(\eta_i\cdot\OVT_i(\mu_i)),\quad\forall\mu_i\in\Pi_i,i\in[n]$.
		\EndFor
		\State \textbf{Output}: $\widehat{\pi}\sim\text{Unif}(\{\prod_{i\in[n]}\mu^1_{i},\cdots,\prod_{i\in[n]}\mu^K_{i}\})$.
	\end{algorithmic}
\end{algorithm}

\begin{remark}
Algorithm~\ref{alg:MDPS self} is also a decentralized algorithm since every agent runs their local algorithm independently without coordination. The only step that requires centralized control is the output process where all the agents need to share the same iteration index, which is also required in the existing decentralized algorithms \citep{mao2021on,jin2021v}.
\end{remark}
\section{Theoretical Guarantees} 
\label{sec:theory}
In this section we analyze the theoretical performance of \mainalg in decentralized policy learning and self-play setting. We first introduce a new complexity measure for function classes and policy classes, called Bellman Evaluation Eluder (BEE) dimension, and then illustrate the regret and sample complexity bounds based on this new measure.

\subsection{Bellman Evaluation Eluder Dimension}
Motivated from Bellman Eluder (BE) dimension in classic MDPs and its variants in MGs \citep{jin2021bellman,jin2021power,huang2021towards}, we propose a new measure specifically tailored to the decentralized policy learning setting, called Bellman Evaluation Eluder (BEE) dimension. First, for any function class $\FC$, we define $(\I-\T^{\Pi,\Pi'}_h)\FC$ to be the Bellman residuals induced by the policies in $\Pi$ and $\Pi'$:
 \begin{align*}
 \label{eq:def bee}
 (\I-\T^{\Pi,\Pi'}_h)\FC:=\{f_h-\T^{\mu,\nu}_hf_{h+1}:f\in\FC,\mu\in\Pi,\nu\in\Pi'\}.
 \end{align*}
 Then Bellman Evaluation Eluder (BEE) dimension is the DE dimension of the Bellman residuals induced by the policy class $\Pi$ and $\Pi'$ on function class $\FC$:
 \begin{definition}
 	\label{def:BEE}
 The $\eps$-Bellman Evaluation Eluder dimension of function class $\FC$ on distribution family $\Q$ with respect to the policy class $\Pi\times\Pi'$ is defined as follows:
 \begin{equation*}
 \DBE(\FC,\eps,\Pi,\Pi',\Q):=\max_{h\in[H]}\DDE((\I-\T^{\Pi,\Pi'}_h)\FC,\Q_h,\eps).
 \end{equation*}
 \end{definition}

BEE dimension is able to capture the generalization error of evaluating value function $V^{\mu\times\nu}$ where $\mu\in\Pi,\nu\in\Pi'$, which is one of the most essential tasks in decentralized policy space optimization as shown in \mainalg. Similar to \cite{jin2021bellman,jin2021power}, we mainly consider two distribution families for $\Q$: \begin{itemize}
\item $\Q^1=\{\Q^1_h\}_{h\in[H]}$: the collection of all probability measures over $\mathcal{S}\times\mathcal{A}\times\B$ at each step when executing $(\mu,\nu)\in\Pi\times\Pi'$.
\item $\Q^2=\{\Q^2_h\}_{h\in[H]}$: the collection of all probability measures that put measure 1 on a single state-action pair $(s,a,b)$ at each step.
\end{itemize}

We also use $\DBE(\FC,\eps,\Pi,\Pi')$ to denote $\min\{\DBE(\FC,\eps,\Pi,\Pi',\Q^1),\DBE(\FC,\eps,\break\Pi,\Pi',\Q^2)\}$ for simplicity in the following discussion.

\paragraph{Relation with Eluder dimension.}
To illustrate the generality of BEE dimension, we show that all function classes with low Eluder dimension also have low BEE dimension, as long as completeness (Assumption~\ref{ass:complete self}) is satisfied. More specifically, we have the following proposition and its  proof is deferred to Appendix~\ref{proof prop BEE}:
\begin{proposition}
\label{prop:BEE}
Assume $\FC$ satisfies completeness, i.e., $\T^{\mu,\nu}_hf_{h+1}\in\FC_h,\forall f\in\FC,\mu\in\Pi,\nu\in\Pi',h\in[H]$. Then for all $\eps>0$, we have
\begin{equation}
\label{eq:prop 1}
\DBE(\FC,\eps,\Pi,\Pi')\leq\max_{h\in[H]}\DE(\FC_h,\eps).
\end{equation}
\end{proposition}
Inequality (\ref{eq:prop 1}) shows that BEE dimension is always upper bounded by Eluder dimension when completeness is satisfied. With Proposition~\ref{prop:BEE}, Appendix~\ref{sec:example} validates that kernel Markov games (including tabular Markov games and linear Markov games) and generalized linear complete models all have small Bellman Evaluation Eluder Dimension. Furthermore, in this case the upper bound of BEE dimension does not depend on $\Pi$ and $\Pi'$, which is a desirable property when $\Pi$ and $\Pi'$ is large.

\subsection{Decentralized Policy Learning Regret}
\label{sec:main result}
Next we present the regret analysis for \mainalg in decentralized policy learning setting. Notice that when $\Pi$ is infinite, the lower bound in \cite{liu2022learning} indicates that the regret will scale with $|\Pi'|$ in tabular cases, suggesting the hardness of efficient learning for infinite policy class $\Pi$. Therefore we focus on finite $\Pi$ here:
\begin{assumption}[Finite player's policy class]
\label{ass:player}
We assume $\Pi$ is finite.
\end{assumption}

We consider two cases, the oblivious opponent (i.e., the opponent determines $\{\nu^t\}_{t=1}^K$ secretly before the game starts) and the adaptive opponent (i.e., the opponent determines its policy adaptively as the game goes on) separately. The difference between these two cases lies in the policy evaluation step of \mainalg. The policy $\nu^t$ of an oblivious opponent does not depend on the collected dataset $\D_{1:t-1}$ and thus $V^{\mu,\nu^t}$ is easier to evaluate. However, for an adaptive opponent, $\nu^t$ will be chosen adaptively based on $\D_{1:t-1}$ and we need to introduce an additional union bound over $\Pi'$ when analyzing the evaluation error of $V^{\mu,\nu^t}$. 

\paragraph{Oblivious opponent.} To attain accurate value function estimation and thus low regrets, we first need to introduce two standard assumptions, realizability and generalized completeness, on $\FC$ and $\GC$ \citep{jin2021bellman,jin2021power}. Here realizability refers to that all the ground-truth action value functions belong to $\FC$ and generalized completeness means that $\GC$ contains all the results of applying Bellman operator to the functions in $\FC$.
\begin{assumption}[Realizability and generalized completeness]
	\label{ass:realize}
	Assume that for any $h\in[H],\mu\in\Pi,\nu\in\{\nu^1,\cdots,\nu^K\},f_{h+1}\in\FC_{h+1}$, we have $Q^{\mu\times\nu}_{h}\in\FC_h,\T_h^{\mu,\nu}f_{h+1}\in\GC_{h}$.
\end{assumption}	

\begin{remark}
Some existing works \citep{xie2021bellman,huang2021towards} assume the completeness assumption, which can also be generalized to our setting:
\begin{assumption}
	\label{ass:complete self}
	Assume that for any $ h\in[H],\mu\in\Pi,\nu\in\Pi',f_{h+1}\in\FC_{h+1}$, we have $\T_h^{\mu,\nu}f_{h+1}\in\FC_{h}$.
\end{assumption}
We want to clarify that Assumption~\ref{ass:complete self} is stronger than generalized completeness in Assumption~\ref{ass:realize} since if Assumption~\ref{ass:complete self} holds, we can simply let $\GC=\FC$ to satisfy generalized completeness.
\end{remark}

Appendix~\ref{sec:realize} shows that realizability and generalized completeness are satisfied in many examples including tabular MGs, linear MGs and kernel MGs with proper function classes. With the above assumptions, we have Theorem~\ref{thm:oblivious} to characterize the regret of \mainalg when the opponent is oblivious, whose proof sketch is deferred to Section~\ref{sec:proof sketch}. To simplify writing, we use the following notations in Theorem~\ref{thm:oblivious}:
\begin{align*}
	\dbee:=\DBE\big(\mathcal{F},\sqrt{1/ K },\Pi,\Pi'\big),\quad\ncov:=\N_{\FC\cup\GC}(\VB/K)KH.
\end{align*}
\begin{theorem}[Regret of Oblivious Adversary]
\label{thm:oblivious}
Under Assumption~\ref{ass:player},\ref{ass:realize}, there exists an absolute constant $c$ such that for any $\delta\in(0,1]$, $K\in\mathbb{N}$, if we choose $\beta=c\VB^2\log(\ncov|\Pi|/\delta)$ and $\eta=\sqrt{\log|\Pi|/(K\VB^2)}$ in \mainalg, then with probability at least $1-\delta$, we have:
\begin{align}
\label{eq:thm1}
\text{Regret}(K)\leq\BO\big(H\VB\sqrt{K\dbee\log\left(\ncov|\Pi|/\delta\right)}\big).
\end{align}
\end{theorem}

 The $\sqrt{K}$ bound on the regret in Theorem~\ref{thm:oblivious} is consistent with the rate in tabular case \citep{liu2022learning} and suggests that the uniform mixture of the output policies $\{\mu^t\}_{t=1}^K$ is an $\eps$-approximate best policy in hindsight when $K=\TO(1/\eps^2)$. The complexity of the problem affects the regret bound through the covering number and the BEE dimension, implying that BEE dimension indeed captures the essence of this problem. Further, in oblivious setting, the regret bound in (\ref{eq:thm1}) does not depend on $\Pi'$ directly (the upper bound of the BEE dimension is also independent of $\Pi'$ in some special cases as shown in Proposition~\ref{prop:BEE}) and thus Theorem~\ref{thm:oblivious} can still hold when $\Pi'$ is infinite, as long as Assumptions~\ref{ass:realize} is satisfied.

\paragraph{Adaptive Opponent.}
In the adaptive setting, the analysis in the oblivious setting can still work but requires slight modifications. We first need to modify Assumption~\ref{ass:realize} to hold for all $\nu\in\Pi'$ since $\nu^t$ is no longer predetermined:
\begin{assumption}[Uniform realizability and generalized completeness]
	\label{ass:realize adapt}
	Assume that for any $h\in[H],\mu\in\Pi,\nu\in\Pi',f_{h+1}\in\FC_{h+1}$, we have $Q^{\mu\times\nu}_{h}\in\FC_h,\T_h^{\mu,\nu}f_{h+1}\in\GC_{h}$.
\end{assumption}	

Further, as we have mentioned before, we need to introduce a union bound over the policies in $\Pi'$ in our analysis and thus we also assume $\Pi'$ to be finite for simplicity.
\begin{assumption}[Finite opponent's policy class]
\label{ass:opponent}
We assume $\Pi'$ is finite.
\end{assumption}
\begin{remark}
When $\Pi'$ is infinite, it is straightforward to generalize our analysis by replacing $|\Pi'|$ with the covering number of $\Pi'$. However, the regret will still depend on the size of $\Pi'$, which is not the case in tabular setting \cite{liu2022learning}. This dependency originates from our model-free type of policy evaluation algorithm (Algorithm~\ref{alg:GOLF}) and thus is inevitable for \mainalg in general. That said, when the Markov game has special structures (e.g., see the Markov game in Section~\ref{sec:cmdp} and Section~\ref{sec:vmdp}), we can avoid this dependency.
\end{remark}

With the above assumptions, we have Theorem~\ref{thm:adaptive} to show that \mainalg can still achieve sublinear regret in adaptive setting, whose proof is deferred to Section~\ref{sec:proof sketch}: 
\begin{theorem}[Regret of Adaptive Adversary]
	\label{thm:adaptive}
	Under Assumption~\ref{ass:player},\ref{ass:realize adapt},\ref{ass:opponent}, there exists an absolute constant $c$ such that for any $\delta\in(0,1]$, $K\in\mathbb{N}$, choosing $\beta=c\VB^2\log(\ncov|\Pi||\Pi'|/\delta)$ and $\eta=\sqrt{\log|\Pi|/(K\VB^2)}$ in \mainalg, then with probability at least $1-\delta$ we have:
	\begin{align}
		\label{eq:thm2}
		\text{Regret}(K)\leq\BO\big(H\VB\sqrt{K\dbee\log\left(\ncov|\Pi||\Pi'|/\delta\right)}\big).
	\end{align}
\end{theorem}
We can see that in adaptive setting the regret also scales with $\sqrt{K}$, implying that \mainalg can still find an $\eps$-approximate best policy in hindsight with $\TO(1/\eps^2)$ episodes even when the opponent is adaptive. Compared to Theorem~\ref{thm:oblivious}, Theorem~\ref{thm:adaptive} has an additional $\log|\Pi'|$ in the upper bound (\ref{eq:thm2}), which comes from the union bound over $\Pi'$ in the analysis. 

\paragraph{Intuitions on the regret bounds.} The regrets in Theorem~\ref{thm:oblivious} and Theorem~\ref{thm:adaptive} can be decomposed to two parts, the online learning error incurred by Hedge and the cumulative value function estimation error incurred by \pealg. From the online learning literature \citep{hazan2016introduction}, the online learning error is $\BO(\VB\sqrt{K\log|\Pi|})$ by viewing the policies in $\Pi$ as experts and $\OVT(\mu)$ as the reward function of expert $\mu$. For the estimation error, we utilize BEE dimensions to bridge $\OVT(\mu^t)-V_1^{\pi^t}(s_1)$ with the function's empirical Bellman residuals on $\D_{1:t-1}$. This further incurs $\TO(\VB\sqrt{K\dbee})$ in the results. Our technical contribution mainly lies in bounding the cumulative value function estimation error with the newly proposed BEE dimensions, which is different from \cite{jin2021bellman} where they focus on bounding the cumulative distance from the optimal value function.


\paragraph{Comparison with existing works.} There have been works studying decentralized policy learning. However, most of them (e.g., \cite{tian2021online,jin2021power,huang2021towards}) only compete  against the Nash value in a two-player zero-sum games, which is a much weaker baseline than ours. \cite{liu2022learning} can achieve a  $\sqrt{K}$ regret upper bound  under Definition~\ref{def:regret}, but their theory is restricted to the  tabular case and  seems unable to deal with more complicated cases. For example, when applying the algorithm in \citep{liu2022learning} to a linear MG, the regret scales with $|\SC|$ and $|\A|$, which becomes vacuous in the face of large state and action space. However, for the case of a linear MG,  \mainalg can achieve a regret bound that depends on the size of the state-action space through the dimension $d$,   rather than $|\SC|$ and $|\A|$.  Thus  \mainalg is able to handle  large or even infinite state and action space. In summary, \mainalg  can achieve a $\sqrt{K}$ regret under Definition~\ref{def:regret} with general function approximation, capable of tackling all models with low BEE dimension,  including linear MGs, kernel MGs and generalized linear complete models (Appendix~\ref{sec:example}).


\subsection{Self-Play Sample Complexity}
\label{sec:selfplay}
Our previous discussion assumes the opponent is arbitrary or even adversary. A natural question is to ask whether there are any additional guarantees if the player and opponent run \mainalg simultaneously, which is exactly Algorithm~\ref{alg:MDPS self} in the self-play setting. The following corollary answers this question affirmatively and shows that Algorithm~\ref{alg:MDPS self} can find an approximate CCE $\widehat{\pi}$ efficiently:
\begin{corollary}
	\label{cor:selfplay}
	Suppose Assumption~\ref{ass:player},\ref{ass:realize adapt} hold for all the agents $i$ and its corresponding $\FC_i,\GC_i,\Pi_i,\Pi_{-i}$. Then for any $\delta\in(0,1],\eps>0$, if we choose
	\begin{align}
		\label{eq:cor1} 
		K\geq\BO\bigg(H^2\VB^2 \cdot \max_{i\in[n]} \biggl\{ \dbeei \cdot \biggl(\log\ncovi+\sum_{j=1}^n\log|\Pi_j|+\log(n/ \delta)\biggr) \biggr\}\bigg/ \eps^2\bigg),
	\end{align} 
	where $\dbeei$ and  $\ncovi$ are defined respectively as 
	\begin{align*}
		\dbeei:=\DBE\big(\mathcal{F}_i,\sqrt{ 1 / K},\Pi_i,\Pi_{-i}\big),\quad\ncovi:=\N_{\FC_i\cup\GC_i}(\VB/K)KH,
	\end{align*}
	and set $\beta_i=c\VB^2\log(\ncovi|\Pi_i||\Pi_{-i}|n/\delta),\eta_i=\sqrt{\log|\Pi_i|/(K\VB^2)}$, then with probability at least $1-\delta$, $\widehat{\pi}$ is $\eps$-approximate CCE.
\end{corollary}

The proof is deferred to Appendix~\ref{proof corollary selfplay}. Corollary~\ref{cor:selfplay} shows that if we run \mainalg independently for each agent, we are able to find an $\eps$-approximate CCE with $\TO(1/\eps^2)$ samples. This can be regarded as a counterpart in Markov games to the classic connection between no-regret learning algorithms and equilibria in matrix games. This guarantee does not hold if an algorithm can only achieve low regrets with respect to the Nash values, which further validates the significance of \mainalg to achieve low regrets under Definition~\ref{def:regret}. 

\paragraph{Avoiding curse of multiagents.} The sample complexity in (\ref{eq:cor1}) avoids exponential scaling with the number of agents $n$ and only scales with $\max_{i\in[n]}\dbeei$, $\max_{i\in[n]}\ncovi$ and $\sum_{j=1}^n\log|\Pi_j|$, suggesting that statistically Algorithm~\ref{alg:MDPS self} is able to escape the \textit{curse-of-multiagents} problem in the literature \citep{jin2021v}. Nevertheless, the input dimension of functions in $\FC_i$ and $\GC_i$ may scale with the number of the agents linearly, leading to the computational inefficiency of \pealg. We comment that finding computational efficient algorithms is beyond the scope of this paper and we leave it to future works. 

\paragraph{Comparison with existing algorithms.} There have been many works studying how to find equilibria in Markov games. However, most of them are focused on centralized two-player zero-sum games \citep{bai2020provable,xie2020learning,jin2021power,huang2021towards} rather than decentralized algorithms. For decentralized algorithms, existing literature mainly handle with potential Markov games \citep{zhang2021gradient,leonardos2021global,ding2022independent} and two-player zero-sum games \citep{daskalakis2020independent,sayin2021decentralized,wei2021last}. \cite{mao2021on,jin2021v} are able to tackle decentralized multi-agent general-sum Markov games while their algorithms are restricted to tabular cases. Algorithm~\ref{alg:MDPS self}, on the other hand, can deal with more general cases with function approximation and policy classes in multi-agent general-sum games. Furthermore, compared to the above works, \mainalg has an additional advantage of robustness to adversaries since all the benign agents can exploit the opponents and achieve no-regret learning. 

\paragraph{Extensions.} Although Theorem~\ref{thm:oblivious}, Theorem~\ref{thm:adaptive} and Corollary~\ref{cor:selfplay} are aimed at Markov games, \mainalg can be applied to a much larger scope of problems. Two such problems are finding the optimal policy in constrained MDP (CMDP, Section~\ref{sec:cmdp}) and vector-valued MDP (VMDP, Section~\ref{sec:vmdp}). We will investigate these two special yet important problems later and demonstrate how to convert such problems into Markov games with a fictitious opponent by duality, where \mainalg is ready to use.

\section{Extension: Constrained Markov Decision Process}
\label{sec:cmdp}
Although \mainalg is designed to solve Markov games, there are quite a lot of other problems where \mainalg can tackle with small adaptation. In this section we investigate an important scenario in practice called constrained Markov decision process (CMDP). By converting CMDP into a maximin problem via Lagrangian multiplier, we will be able to view it as a zero-sum Markov game and apply \mainalg readily.

\paragraph{Constrained Markov decision process.} Consider the Constrained Markov Decision Process (CMDP) \citep{ding2021provably} $\MC=(\mathcal{S},\mathcal{A},\{P_h\}_{h=1}^H,\{r_h\}_{h=1}^H,\{g_h\}_{h=1}^H,H)$ where $\mathcal{S}$ is the state space, $\mathcal{A}$ is the action space, $H$ is the length of each episode, $P_h: \mathcal{S}\times\mathcal{A}\to\Delta(\mathcal{S})$ is the transition function at $h$-th step, $r_h:\mathcal{S}\times\mathcal{A}\to\R_+$ is the reward function and $g_h:\mathcal{S}\times\mathcal{A}\to[0,1]$ is the utility function at $h$-th step. We assume the reward $r_h$ is also bounded in $[0,1]$ for simplicity and thus $\VB=H$. Then given a policy $\mu=\{\mu_h:\SC\to\Delta_{\A}\}_{h\in[H]}$, we can define the value function $\VRH^{\mu}$ and action-value function $\QRH^{\mu}$ with respect to the reward function $r$ as follows:
\begin{align*}
\VRH^{\mu}(s)=\E_{\mu}\bigg[\sum_{t=h}^H r_t(s_t,a_t)\bigg|s_h=s\bigg],\QRH^{\mu}(s,a)=\E_{\mu}\bigg[\sum_{t=h}^H r_t(s_t,a_t)\bigg|s_h=s,a_h=a\bigg].
\end{align*}
The value function $\VGH^{\mu}$ and action-value function $\QGH^{\mu}$ with respect to the utility function $g$ can be defined similarly. Another related concept is the state-action visitation distribution, which can be defined as
\begin{align*}
	d^{\mu}_{h}(s,a)=\text{Pr}_{\mu}[(s_h,a_h)=(s,a)],
\end{align*}
where $\text{Pr}_{\mu}$ denotes the distribution of the trajectory induced by executing policy $\mu$ in the $\MC$.

\paragraph{Learning objective.} In CMDP, the player aims to solve a constrained problem where the objective function is the expected total rewards and the constraint is on the expected total utilities:
\begin{problem}[Optimization problem of CMDP]
\label{prob:cmdp}
\begin{align}
\label{eq:prob cmdp}
\max_{\mu\in\Pi} V^{\mu}_{r,1}(s_1) \quad \text{subject to}\quad V^{\mu}_{g,1}(s_1)\geq b,
\end{align}
where $b\in(0,H]$ to avoid triviality.
\end{problem}

Denote the optimal policy for (\ref{eq:prob cmdp}) by $\muc$, then the regret can be defined as the performance gap with respect to $\muc$:
\begin{align}
\label{eq:cmdp regret}
	\reg(K)=\sum_{t=1}^K\Big(\VR^{\muc}(s_1)-\VR^{\mu^t}(s_1)\Big).
\end{align}
 
However, since utility information is only revealed after a policy is decided, it is impossible for each policy to satisfy the constraints. Therefore, like \cite{ding2021provably}, we allow each policy to violate the constraint in each episode and focus on minimizing total constraint violations over $K$ episodes:
\begin{align}
\label{eq:violation}
\vio(K)=\bigg[\sum_{t=1}^K\Big(b-\VG^{\mu^t}(s_1)\Big)\bigg]_{+}.
\end{align} 
Achieving sublinear violations in (\ref{eq:violation}) implies that if we sample a policy uniformly from $\{\mu^t\}_{t=1}^K$, its constraint violation can be arbitrarily small given large enough $K$. Therefore, if an algorithm can achieve sublinear regret in (\ref{eq:cmdp regret}) and sublinear violations in (\ref{eq:violation}) at the same time, this algorithm will be able to find a good approximate policy to $\muc$.

\subsection{Algorithm: \cmdpalg}
To solve Problem~\ref{prob:cmdp} with \mainalg, we first need to convert it into a Markov game. A natural idea is to apply the Lagrangian multiplier $Y\in\R_{+}$ to Problem~\ref{prob:cmdp}, which brings about the equivalent maximin problem below:
\begin{align}
\label{prob:maximin}
\max_{\mu\in\Pi}\min_{Y\geq0} \LO(\mu,Y):=V^{\mu}_{r,1}(s_1)+Y(\VG^{\mu}(s_1)-b).
\end{align} 
Although Problem~\ref{prob:cmdp} is non-concave in $\mu$, there have been works indicating that strong duality still holds for Problem~\ref{prob:cmdp} when the policy class is described by a good parametrization \citep{paternain2019constrained}. Therefore, here we assume strong duality holds and it is straightforward to generalize our analysis to the case where there exists a duality gap:

\begin{assumption}[Strong duality]
\label{ass:strong duality}
Assume strong duality holds for Problem~\ref{prob:cmdp}, i.e.,
\begin{align}
\label{eq:strong duality}
\max_{\mu\in\Pi}\min_{Y\geq0} \LO(\mu,Y)=\min_{Y\geq0}\max_{\mu\in\Pi} \LO(\mu,Y). 
\end{align}
\end{assumption}

\begin{remark}
One example case where strong duality (\ref{eq:strong duality}) holds is when policy class $\Pi$ satisfies global realizability. Let $\mu^*_{\text{glo}}$ denote the solution to $\max_{\mu_h(\cdot|s)\in\Delta_{\A}}\min_{Y\geq0} \LO(\mu,Y)$. \cite{ding2021provably} showed that $\max_{\mu\in(\Delta_{\A})^{|\SC|H}}\min_{Y\geq0} \LO(\mu,Y)$ satisfies strong duality, and thus as long as $\mu^*_{\text{glo}}\in\Pi$, Problem~\ref{prob:cmdp} also has strong duality.
\end{remark}

Further, let $D(Y):=\max_{\mu\in\Pi}\LO(\mu,Y)$ denote the dual function and suppose the optimal dual variable is $Y^*=\arg\min_{Y\geq0}D(Y)$. To ensure $Y^*$ is bounded, we need to assume that the standard Slater's Condition holds:
\begin{assumption}
\label{ass:slater}
There exists $\sla>0$ and $\tmu\in\Pi$ such that $\VG^{\tmu}(s_1)\geq b+\sla$.
\end{assumption}

Then the following lemma shows that Assumption~\ref{ass:slater} implies bounded optimal dual variable, whose proof is deferred to Appendix~\ref{proof lemma duality}:
\begin{lemma}
\label{lem:duality}
Suppose Assumption~\ref{ass:strong duality},\ref{ass:slater} hold, then we have $0\leq Y^*\leq {H}/{\sla}$.
\end{lemma}

Now we are ready to adapt \mainalg into a primal-dual algorithm to solve Problem~\ref{prob:cmdp}. Notice that the maximin problem (\ref{prob:maximin}) can be viewed as a zero-sum Markov game where the player's policy is $\mu$ and the reward function for the player is $r_{h}(s,a)+Yg_h(s,a)$. The opponent's action is $Y\in\R_+$ which remains the same throughout a single episode. With this formulation, we can simply run \mainalg on the player, assuming the player is given function classes $\{\FCR,\GCR\}$ and $\{\FCG,\GCG\}$ to approximate $\QRH^{\mu}$ and $\QGH^{\mu}$ respectively. In the meanwhile, we run online projected gradient descent on the opponent so that its action $Y$ can capture the total violation so far.

This new algorithm is called \cmdpalg and shown in Algorithm~\ref{alg:CMDP}. It consists of the following three steps in each iteration. For the policy evaluation task in the second step, \cmdpalg runs a single-agent version of \pealg to estimate $V^{\mu}_{r,1}(s_1)$ and $V^{\mu}_{g,1}(s_1)$ separately, which is essential for \cmdpalg to deal with the infinity of the opponent's policy class, i.e., $\R_+$.
\begin{itemize}
	\item The player plays a policy $\mu^t$ sampled from its hyperpolicy $p^t$ and collects a trajectory.
	\item The player runs \spealg to obtain optimistic value function estimations $\OVTR(\mu),\break\OVTG(\mu)$ for all $\mu\in\Pi$ and updates the hyperpolicy using Hedge with the loss function being $\OVTR(\mu)+Y_t\OVTG(\mu)$. The construction rule for $\B_{\D}(\mu)$ is still based on relaxed least-squared policy evaluation:
	\begin{align}
	\label{eq:cmdp eva}
		 \B_{\D}(\mu)\gets\{f\in\FC:\LC_{\D}(f_h,f_{h+1},\mu)\leq\inf_{g\in\GC}\LC_{\D}(g_h,f_{h+1},\mu)+\beta,\forall h\in[H]\},
        \end{align}
    where $\LC_{\D}$ is the empirical Bellman residuals on $\D$:
\begin{align*} \LC_{\D}(\xi_h,\zeta_{h+1},\mu)=\sum_{(s_h,a_h,x_h,s_{h+1})\in\D}[\xi_h(s_h,a_h)-x_h-\zeta_{h+1}(s_{h+1},\mu)]^2.
\end{align*}

	\item The dual variable is updated using online projected gradient descent.
\end{itemize}

\begin{algorithm}[h]
	\caption{\textbf{\cmdpalg}}
	\label{alg:CMDP}
	\begin{algorithmic}
		\State \textbf{Input}: learning rate $\eta,\alpha$, confidence parameter $\beta_r,\beta_g$, projection length $\rchi$.
		\State Initialize $p^1\in\R^{|\Pi|}$ to be uniform over $\Pi$, $Y_1\gets 0$.
		\For{$t=1,\cdots,K$}
		\State \textbf{Collect samples:}
		\State The player samples $\mu^t$ from $p^t$.
		\State Run $\mu^t$ and collect $\D^r_{t}=\{s^t_1,a^t_1,r^t_1,\cdots,s^t_{H+1}\}$,$\D^g_{t}=\{s^t_1,a^t_1,g^t_1,\cdots,s^t_{H+1}\}$.
		\State \textbf{Update policy distribution:}
		\State $\OVTR(\mu)\gets\textbf{\spealg}(\mu,\D^r_{1:t-1},\FCR,\GCR,\beta_r),\quad\forall\mu\in\Pi$.
		\State
	    $\OVTG(\mu)\gets\textbf{\spealg}(\mu,\D^g_{1:t-1},\FCG,\GCG,\beta_g),\quad\forall\mu\in\Pi$.
		\State $p^{t+1}(\mu)\propto p^{t}(\mu)\cdot\exp(\eta\cdot(\OVTR(\mu)+Y_t\OVTG(\mu))),\quad\forall\mu\in\Pi$.
		\State \textbf{Update dual variable:}
		\State 
		$Y_{t+1}\gets\proj(Y_t+\alpha(b-\OVTG(\mu^t)))$.		
		\EndFor
	\end{algorithmic}
\end{algorithm}

\begin{algorithm}[h]
	\caption{$\textbf{\spealg}(\mu,\D,\FC,\GC,\beta)$}
	\label{alg:SGOLF}
	\begin{algorithmic}
		\State \textbf{Construct} $\B_{\D}(\mu)$ based on $\D$ via (\ref{eq:cmdp eva}).
		\State \textbf{Select} $\bar{V}\gets\max_{f\in\B_{\D}(\mu)} f(s_1,\mu)$.
		\State \Return $\bar{V}$.
	\end{algorithmic}
\end{algorithm}

\subsection{Theoretical Guarantees}
Next we provide the regret and constraint violation bounds for \cmdpalg. Here we also consider the case where $\Pi$ is finite, i.e., Assumption~\ref{ass:player} is true. However, we can see that here the opponent is adaptive and its policy class is infinite, suggesting that Assumption~\ref{ass:opponent} is violated. Fortunately, since the opponent only affects the reward function, the player can simply first estimate $V^{\mu}_{r,1}(s_1)$ and $V^{\mu}_{g,1}(s_1)$ respectively and then use their weighted sum to approximate the target value function $V^{\mu}_{r,1}(s_1)+Y\cdot V^{\mu}_{g,1}(s_1)$. In this way, \cmdpalg circumvents introducing a union bound on $Y$ and thus can work even when the number of possible values for $Y$ is infinite.  

We also need to introduce the realizability and general completeness assumptions on the function classes as before:
\begin{assumption}[Realizability and generalized completeness in CMDP]
	\label{ass:realize single}
	Assume that for any $h\in[H],\mu\in\Pi,f^{r}_{h+1}\in\FCR_{h+1},f^{g}_{h+1}\in\FCG$, we have
	\begin{align}
		\label{eq:realize cmdp}
		\QRH^{\mu}\in\FCR_h,\QGH^{\mu}\in\FCG_h, \T_h^{\mu,r}f^r_{h+1}\in\GCR_{h}, \T_h^{\mu,g}f^g_{h+1}\in\GCG_{h}.
	\end{align}
\end{assumption}	
Here $\T_h^{\mu,r}$ is the Bellman operator at step $h$ with respect to $r$:
\begin{align*}
	(\T_h^{\mu,r}f_{h+1})(s,a)=r_h(s,a)+\E_{s'\sim P(\cdot|s,a)}f_{h+1}(s',\mu),
\end{align*}
where $f_{h+1}(s',\mu)=\E_{a'\sim\mu(\cdot|s)}[f_{h+1}(s',a')]$. $\T_h^{\mu,g}$ is defined similarly. We can see that (\ref{eq:realize cmdp}) simply says that all the action value functions with respect to $r$ ($g$) belong to $\FCR$ ($\FCG$) and $\GCR$ ($\GCG$) contains all the results of applying Bellman operator with respect to $r$ ($g$) to the functions in $\FCR$ ($\FCG$).


In addition, as a simplified case of Definition~\ref{def:BEE}, BEE dimension for single-agent setting can be defined as follows:
\begin{definition}
	The single-agent $\eps$-Bellman Evaluation Eluder dimension of function class $\FC$ on distribution family $\Q$ with respect to the policy class $\Pi$ and the reward function $r$ is defined as follows:
	\begin{equation*}
		\DBE(\FC,\eps,\Pi,r,\Q):=\max_{h\in[H]}\DDE((\I-\T^{\Pi,r}_h)\FC,\Q_{h},\eps),
	\end{equation*}
	where $(\I-\T^{\Pi,r}_h)\FC:=\{f_h-\T^{\mu,r}_hf_{h+1}:f\in\FC,\mu\in\Pi\}$.
\end{definition}

We also let  $\DBE(\FC,\eps,\Pi,r)$  denote $\min\{\DBE(\FC,\eps,\Pi,r,\Q^1),\DBE(\FC,\eps,\break\Pi,r,\Q^2)\}$ as before. $\DBE(\FC,\eps,\Pi,g,\Q)$ and $\DBE(\FC,\eps,\Pi,g)$ are defined similarly but with respect to the utility function $g$.

Now we can present the following theorem which shows that \cmdpalg is capable of achieving sublinear regret and constraint violation for Problem~\ref{prob:cmdp}. We also use the following notations to simplify writing:
\begin{align*}
&\dbeer:=\DBE\big(\FCR,\sqrt{{1}/{K}},\Pi,r\big),\quad\ncovr:=\N_{\FCR\cup\GCR}(H/K)KH,\\
&\dbeeg:=\DBE\big(\FCG,\sqrt{{1}/{K}},\Pi,g\big),\quad\ncovr:=\N_{\FCG\cup\GCG}(H/K)KH.
\end{align*}

\begin{theorem}
\label{thm:cmdp}
Under Assumption~\ref{ass:strong duality},\ref{ass:slater},\ref{ass:player},\ref{ass:realize single}, there exists an absolute constant $c$ such that for any $\delta\in(0,1]$, $K\in\mathbb{N}$, if we choose $\beta_r=cH^2\log(\ncovr|\Pi|/\delta)$, $\beta_g=cH^2\log(\ncovg|\Pi|/\delta)$, $\alpha=1/{\sqrt{K}}$, $\rchi={2H}/{\sla}$ and $\eta=\sqrt{{\log|\Pi|}/({K(\rchi+1)^2H^2})}$ in \cmdpalg, then with probability at least $1-\delta$, we have:
\begin{align}
	\label{eq:thm3 regret}
	&\text{Regret}(K)\leq\BO\bigg(\bigg(H^2+\frac{H^2}{\sla}\bigg)\sqrt{K\dbeer\log\left(\ncovr|\Pi|/\delta\right)}\bigg),
\end{align}
\begin{align}
	\label{eq:thm3 violation}
	\text{Violation}(K)\leq\BO\bigg(\bigg(H^2+\frac{H}{\sla}\bigg)\sqrt{K\ebe}\bigg),
\end{align}
where we define $\ebe$ as 
\begin{align*} \ebe=\max\Big\{\dbeer\log\left(\ncovr|\Pi|/\delta\right),\dbeeg\log\left(\ncovg|\Pi|/\delta\right)\Big\}.
\end{align*}
\end{theorem}

The bounds in (\ref{eq:thm3 regret}) and (\ref{eq:thm3 violation}) show that both the regret and constraint violation of \cmdpalg scale with $\sqrt{K}$. This implies that for any $\eps>0$, if $\widehat{\mu}$ is sampled uniformly from $\{\mu^t\}_{t=1}^K$ and $K\geq\TO(1/\eps^2)$, $\widehat{\mu}$ will be an $\eps$ near-optimal policy with high probability in the sense that
\begin{align*}
	\VR^{\widehat{\mu}}(s_1)\geq \VR^{\muc}(s_1)-\eps,\qquad \VG^{\widehat{\mu}}(s_1)\geq b-\eps.
\end{align*} 

In addition, compared to the results in Theorem~\ref{thm:oblivious} and Theorem~\ref{thm:adaptive}, (\ref{eq:thm3 regret}) and (\ref{eq:thm3 violation}) have an extra term scaling with $1/\sla$. This is because \cmdpalg is a primal-dual algorithm and $\sla$ characterizes the regularity of this constrained optimization problem.

The proof of the regret bound is similar to Theorem~\ref{thm:oblivious} and Theorem~\ref{thm:adaptive} by viewing $V^{\mu}_{r,1}(s_1)+YV^{\mu}_{g,1}(s_1)$ as the target value function and decomposing the regret into cumulative estimation error and online learning error. To bound the constraint violation, we need to utilize the strong duality and the property of online projected gradient descent. See Appendix~\ref{proof:thm cmdp} for more details.  

\paragraph{Comparison with existing algorithms.} There has been a line of works studying the exploration and exploitation in CMDPs. \cite{efroni2020exploration,ding2021provably} propose a series of algorithms which can achieve $\sqrt{K}$ bound on regrets and constraint violations. However, they focus on tabular cases or linear function approximation and do not consider policy classes while \cmdpalg can deal with nonlinear function approximation and policy classes. As an interesting follow-up, \cite{liu2021learning} reduces the constraint violation to $\TO(1)$ by adding slackness to the algorithm and achieves zero violation when a strictly safe policy is known; \cite{wei2021provably} further avoids such requirement with the price of worsened regrets. However, these improvements are all limited in tabular cases and we leave the consideration of their general function approximation counterpart to future works. 
\section{Extension: Vector-valued Markov Decision process}
\label{sec:vmdp}
Another setting where \mainalg can play a role is the approachability task for vector-valued Markov decision process (VMDP) \citep{miryoosefi2019reinforcement,yu2021provably,miryoosefi2021simple}. Similar to CMDP, we convert it into a zero-sum Markov game by Fenchel's duality and then adapt \mainalg properly to solve it.

\paragraph{Vector-valued Markov decision process.} Consider the Vector-valued Markov decision process (VMDP) \citep{yu2021provably} $\MV=(\mathcal{S},\mathcal{A},\{P_h\}_{h=1}^H,\br,H)$ where $\br=\{\br_h:\mathcal{S}\times\mathcal{A}\to[0,1]^d\}_{h=1}^H$ is a collection of $d$-dimensional reward functions and the rest of the components are defined the same as in Section~\ref{sec:cmdp}. Then given a policy $\mu\in\Pi$, we can define the corresponding $d$-dimensional value function $\bv^{\mu}_h:\SC\to[0,H]^d$ and action-value function $\bq^{\mu}_h:\SC\times\A\to[0,H]^d$ as follows:
\begin{align*}
&\bv_h^{\mu}(s)=\E_{\mu}\bigg[\sum_{t=h}^H \br_t(s_t,a_t)\bigg|s_h=s\bigg],\quad\bq_h^{\mu}(s,a)=\E_{\mu}\bigg[\sum_{t=h}^H \br_t(s_t,a_t)\bigg|s_h=s,a_h=a\bigg].
\end{align*}

\paragraph{Learning objective.} In this paper we study the approachability task \citep{miryoosefi2019reinforcement} in VMDP where the player needs to learn a policy whose expected cumulative reward vector lies in a convex target set $\CS$. We consider a more general agnostic version \citep{yu2021provably,miryoosefi2021simple} where we do not assume the existence of such policies and the player learns to minimize the Euclidean distance between expected reward and the target set $\CS$:
\begin{problem}[Approachability for VMDP]
	\label{prob:vmdp}
	\begin{align*}
		\min_{\mu\in\Pi} \dist(\bv^{\mu}_1(s_1),\CS),
	\end{align*}
	where $\dist(\bx,\CS)$ is the Euclidean distance between point $\bx$ and set $\CS$.
\end{problem}

The approachability for VMDP is a natural objective in multi-task reinforcement learning where each dimension of the reward can be regarded as a task. It is important in many practical domains such as robotics, autonomous vehicles and recommendation systems \citep{yu2021provably}. Therefore, finding the optimal policy for Problem~\ref{prob:vmdp} efficiently is of great significance in modern reinforcement learning.

\subsection{Algorithm: \vmdpalg}
To deal with Probelm~\ref{prob:vmdp}, we first convert Problem~\ref{prob:vmdp} into a Markov game as we have done in Section~\ref{sec:cmdp}. By Fenchel's duality of the distance function, we know Problem~\ref{prob:vmdp} is equivalent to the following minimax problem:
\begin{align*}
\min_{\mu\in\Pi}\max_{\bth\in\BL(1)} \LV(\mu,\bth):=\langle \bth,\bv^{\mu}_1(s_1)\rangle-\max_{\bx\in\CS}\langle\bth,\bx\rangle,
\end{align*} 
where $\BL(r)$ is the $d$-dimensional Euclidean ball of radius $r$ centered at the origin. Regarding $\mu$ as the player's policy and $\bth$ as the opponent, we can again view this minimax problem as a Markov game where the reward function for the player is $\langle \bth,\br_h(s,a)\rangle$. Consider the general function approximation case that the player is given function classes $\FC:=\{\FC^j_{h}\}_{h,j=1}^{H,d},\GC:=\{\GC^j_{h}\}_{h,j=1}^{H,d}$ to approximate $\bq_h^{\mu}$ where $\FC^j_h$ and $\GC^j_h$ are the $j$-th dimension of $\FC_h$ and $\GC_h$, then we can run \mainalg for the player while the opponent will update $\bth$ with online projected gradient ascent just like \cmdpalg.

We call this new algorithm \vmdpalg, which is shown in Algorithm~\ref{alg:VMDP} and also consists of three steps in each iteration. For the policy evaluation task here, we apply \vpealg and construct a confidence set for each dimension of the function class separately, and let the final confidence set be their intersection. Therefore the construction rule for $\B_{\D}(\mu)$ is given as:
\begin{align}
\label{eq:vmdp eva}
 \B_{\D}(\mu)\gets\{f\in\FC:\LC_{\D^j}(f^j_{h},f^j_{h+1},\mu)\leq\inf_{g\in\GC}\LC_{\D^j}(g^j_{h},f^j_{h+1},\mu)+\beta,\forall h\in[H],j\in[d]\},
\end{align}
where for any $j\in[d]$ and $h\in[H]$,
\begin{align*} \LC_{\D^j}(\xi^j_{h},\zeta^j_{h+1},\mu)=\sum_{(s_h,a_h,r^j_{h},s_{h+1})\in\D}[\xi^j_{h}(s_h,a_h)-r^j_{h}-\zeta^j_{h+1}(s_{h+1},\mu)]^2,
\end{align*}
and $r^j_{h}$ is the $j$-the dimension of $\br_h$. In addition, since here we want to minimize the distance, \vpealg will output a pessimistic estimate of the target value function instead of an optimistic one.
 \begin{itemize}
 	\item The player plays a policy $\mu^t$ sampled from its hyperpolicy $p^t$ and collects a trajectory.
 	\item The player runs \vpealg to obtain pessimistic value function estimations $\langle\bth^t,\PET(\mu)\rangle$ for all $\mu\in\Pi$ and updates the hyperpolicy using Hedge.
 	\item The dual variable is updated using online projected gradient ascent.
 \end{itemize}

\begin{algorithm}[h]
	\caption{\textbf{\vmdpalg}}
	\label{alg:VMDP}
	\begin{algorithmic}
		\State \textbf{Input}: learning rate $\eta,\alpha_t$, confidence parameter $\beta$.
		\State Initialize $p^1\in\R^{|\Pi|}$ to be uniform over $\Pi$, $\bth_1\gets 0$.
		\For{$t=1,\cdots,K$}
		\State \textbf{Collect samples:}
		\State The learner samples $\mu^t$ from $p^t$.
		\State Run $\mu^t$ and collect $\D_{t}=\{s^t_1,a^t_1,\br^t_{1},\cdots,s^t_{H+1}\}$. 
		\State \textbf{Update policy distribution:}
		\State $\PET(\mu)\gets\textbf{\vpealg}(\mu,\D_{1:t-1},\FC,\GC,\beta,\bth^t),\quad\forall\mu\in\Pi$.
		\State $p^{t+1}(\mu)\propto p^{t}(\mu)\cdot\exp(-\eta\langle\PET(\mu),\bth^t\rangle),\quad\forall\mu\in\Pi$.
		\State \textbf{Update dual variable:}
		\State 
		$\bth_{t+1}\gets\projb(\bth_t+\alpha_t(\PET(\mu_t)-\arg\max_{\bx\in\CS}\langle\bth_t,\bx\rangle))$.		
		\EndFor
		\State\textbf{Output}: $\widehat{\mu}$ uniformly sampled from $\mu^1,\cdots,\mu^K$.
	\end{algorithmic}
\end{algorithm}

\begin{algorithm}[h]
	\caption{$\textbf{\vpealg}(\mu,\D,\FC,\GC,\beta,\bth)$}
	\label{alg:VGOLF}
	\begin{algorithmic}
		\State \textbf{Construct} $\B_{\D}(\mu)$ based on $\D$ via (\ref{eq:vmdp eva}).
		\State \textbf{Select} $\underline{\boldsymbol{V}}\gets f_1(s_1,\mu)$, where $f=\arg
	\min_{f'\in\B_{\D}(\mu)} \langle f'_1(s_1,\mu),\bth\rangle$.
		\State \Return $\underline{\boldsymbol{V}}$.
	\end{algorithmic}
\end{algorithm}

\subsection{Theoretical Guarantees}
In this subsection, we still consider finite policy class $\Pi$. Notice that in the fictitious MG of VMDP, the  policy class of the opponent is also infinite, i.e., $\BL(1)$. However, since the player only needs to estimate $\bv^{\mu}_1(s_1)$, which is independent of $\bth$, \vmdpalg can also circumvent the union bound on $\bth$ just like \cmdpalg. 

In addition, we need to introduce the realizability and generalized completeness assumptions in this specific setting, which is simply a vectorized version as before:
\begin{assumption}[Realizability and generalized completeness in VMDP]
	\label{ass:realize v}
	Assume that for any $h\in[H],j\in[d],\mu\in\Pi,f_{h+1}\in\FC_{h+1}$, we have $Q^{\mu,j}_{h}\in\FC_{h,j},\T_h^{\mu,j}f^j_{h+1}\in\GC^j_{h}$,
	where $Q_{h}^{\mu,j}$ is the $j$-the dimension of $\bq_h^{\mu}$ and $\T_h^{\mu,j}$ is the $j$-th dimensional Bellman operator at step $h$ defined in (\ref{eq:vmdp bellman}).
\end{assumption}	
Here $\T_h^{\mu,j}$ is defined as:
\begin{align}
\label{eq:vmdp bellman}
	(\T_h^{\mu,j}f^j_{h+1})(s,a):=r^j_{h}(s,a)+\E_{s'\sim P(\cdot|s,a)}f^j_{h+1}(s',\mu).
\end{align}

In addition, the BEE dimension for VMDP can be defined as the maximum BEE dimension among all $d$ dimensions:
\begin{definition}
	The $d$-dimensional $\eps$-Bellman Evaluation Eluder dimension of function class $\FC$ on distribution family $\Q$ with respect to the policy class $\Pi$ is defined as follows:
	\begin{equation*}
		\DBE(\FC,\eps,\Pi,\Q):=\max_{j\in[d],h\in[H]}\DDE((\I-\T^{\Pi,j}_h)\FC^j,\Q_{h},\eps),
	\end{equation*}
	where $(\I-\T^{\Pi,j}_h)\FC^j:=\{f^j_{h}-\T^{\mu,j}_hf^j_{h+1}:f\in\FC,\mu\in\Pi\}$.
\end{definition}
We also use $\DBE(\FC,\eps,\Pi)$ to denote $\min\{\DBE(\FC,\eps,\Pi,\Q^1),\DBE(\FC,\eps,\Pi,\Q^2)\}$ as before.

The next theorem shows that \vmdpalg is able to find a near optimal policy for Problem~\ref{prob:vmdp} with polynomial samples, where we use the following notations to simplify writing:
\begin{align*}
	\dbeev:=\DBE\big(\FC,\sqrt{{1}/{K}},\Pi\big),\quad\ncovv:=\max_{j\in[d]}\N_{\FC^j\cup\GC^j}(H/K)KH.
\end{align*}
\begin{theorem}
\label{thm:vmdp}
Under Assumption \ref{ass:player},\ref{ass:realize v}, there exists an absolute constant $c$ such that for any $\delta\in(0,1]$, $K\in\mathbb{N}$, if we choose $\beta=cH^2\log(\ncovv|\Pi|d/\delta)$, $\alpha_t={2}/({H\sqrt{dt}})$, and $\eta=\sqrt{{\log|\Pi|}/({KH^2d})}$ in \vmdpalg, then with probability at least $1-\delta$, we have:
\begin{align}
\label{eq:thm4}
\dist(\bv^{\widehat{\mu}}_1(s_1),\CS)\leq\min_{\mu\in\Pi}\dist(\bv^{\mu}_1(s_1),\CS)+\BO\Big(H^2\sqrt{d}\cdot\sqrt{{\dbeev\log\left(\ncovv|\Pi|d/\delta\right)}/{K}}\Big).
\end{align}
\end{theorem}

The bound in (\ref{eq:thm4}) shows that for any $\eps>0$, if $K\geq\TO(d/\eps^2)$, $\widehat{\mu}$ will be an $\eps$ near-optimal policy with high probability. Compared to the results in Theorem~\ref{thm:oblivious} and Theorem~\ref{thm:adaptive}, there is an additional term $d$. This is because the reward is $d$-dimensional and we are indeed evaluating $d$ scalar value functions in \vpealg.

The proof is similar to that of Theorem~\ref{thm:cmdp} and utilizes the fact that both $\mu$ and $\bth$ are updated via no-regret online learning algorithms (Hedge for $\mu$ and online projected gradient ascent for $\bth$). See Appendix~\ref{proof:thm vmdp} for more details.

\paragraph{Comparison with existing algorithms.} \cite{yu2021provably} has also proposed algorithms for approachability tasks in tabular cases and achieve the same sub-optimality gap with respect to $d$ and $K$ as Theorem~\ref{thm:vmdp}. \cite{miryoosefi2021simple} studies the tabular and linear approximation cases, achieving $\sqrt{K}$ regret as well. Their sample complexity does not scale with $d$ because they have normalized the reward vector to lie in $\BL(1)$ in tabular cases and $\BL(\sqrt{d_{\mathrm{lin}}})$ in $d_{\mathrm{lin}}$-dimensional linear VMDPs. Compared to the above works, \vmdpalg is able to tackle the more general cases with nonlinear function approximation and policy classes while retaining the sample efficiency.

\section{Proof Sketch of Theorem~\ref{thm:oblivious} and Theorem~\ref{thm:adaptive}}
\label{sec:proof sketch}
In this section we present a proof sketch for Theorem~\ref{thm:oblivious} and Theorem~\ref{thm:adaptive}. We first consider the oblivious setting. Let $\mu^*=\arg\max_{\mu\in\Pi}\sum_{t=1}^{K}V_1^{\mu\times\nu^t}(s_1)$ and we can decompose the regret into the following terms:
\begin{align}
	&\max_{\mu\in\Pi}\sum_{t=1}^{K}V_1^{\mu\times\nu^t}(s_1)-\sum_{t=1}^{K} V_1^{\pi^t}(s_1)\notag\\
	&\qquad=\bigg(\underbrace{\sum_{t=1}^{K}V_1^{\mu^*\times\nu^t}(s_1)-\sum_{t=1}^{K}\OVT(\mu^*)}_{\displaystyle(1)}\bigg) +\bigg(\underbrace{\sum_{t=1}^{K}\OVT(\mu^*)-\sum_{t=1}^{K}\langle\OVT,p^t\rangle}_{\displaystyle(2)}\bigg)\notag\\
	&\qquad\qquad +\bigg(\underbrace{\sum_{t=1}^{K}\langle\OVT,p^t\rangle-\sum_{t=1}^{K}\OVT(\mu^t)}_{\displaystyle(3)}\bigg) +\bigg(\underbrace{\sum_{t=1}^{K}\OVT(\mu^t)-\sum_{t=1}^{K}V^{\pi^t}_1(s_1)}_{\displaystyle(4)}\bigg).\label{eq:decompose}
\end{align}
Our proof bounds these terms separately and mainly consists of three steps:
\begin{itemize}
\item Prove $\OVT(\mu)$ is an optimistic estimation of $V_1^{\mu\times\nu^t}(s_1)$ for all $t\in[K]$ and $\mu\in\Pi$, which implies that term $(1)\leq0$.
\item Bound term (4), the cumulative estimation error $\sum_{t=1}^K\OVT(\mu^t)-V_1^{\pi^t}(s_1)$. In this step we utilize the newly proposed complexity measure BEE dimension to bridge the cumulative estimation error and the empirical Bellman residuals occurred in \pealg.
\item Bound term (2) using the existing results of online learning error induced by Hedge and bound (3) by noticing that it is a martingale difference sequence.
\end{itemize}

\subsection{Step 1: Prove Optimism}
First we can show that the constructed set $\B_{\D_{1:t-1}}(\mu,\nu^t)$ is not vacuous in the sense that the true action-value function $Q^{\mu,\nu^t}$ belongs to it with high probability
\begin{lemma}
\label{lem:optimism}
With probability at least $1-{\delta}/{4}$, we have for all $t\in[K]$ and $\mu\in\Pi$, $Q^{\mu,\nu^t}\in\B_{\D_{1:t-1}}(\mu,\nu^t)$.	
\end{lemma}
\begin{proof}
See Appendix~\ref{proof lemma optimism}.
\end{proof}
Then since $\OVT(\mu)=\max_{f\in\B_{\D_{1:t-1}}(\mu,\nu^t)}f(s_1,\mu,\nu^t)$, we know for all $t\in[K]$ and $\mu\in\Pi$,
\begin{align*}
\OVT(\mu)\geq Q^{\mu,\nu^t}(s_1,\mu,\nu^t)=V_1^{\mu\times\nu^t}(s_1).
\end{align*}
In particular, we have for all $t\in[K]$,
\begin{align}
	\label{eq:optimism}
	\OVT(\mu^*)\geq V_1^{\mu^*\times\nu^t}(s_1).
\end{align}
Thus, \eqref{eq:optimism} implies that $\OVT(\mu^*)$ is an optimistic estimate of $V_1^{\mu^*\times\nu^t}(s_1)$ for all $t$, and therefore   term (1) in \eqref{eq:decompose} is non-positive.

\subsection{Step 2: Bound Estimation Error} 
Next we aim to handle term (4) in  \eqref{eq:decompose} and show the estimation error $\sum_{t=1}^K\OVT(\mu^t)-V_1^{\pi^t}(s_1)$ is small. Let 
$f^{t,\mu}=\arg\max_{f\in\B_{\D_{1:t-1}}(\mu,\nu^t)}f(s_1,\mu,\nu^t)$. Then using standard concentration inequalities, we can have the following lemma which says that empirical Bellman residuals are indeed close to true residuals with high probability. Recall that here $\pi^k=\mu^k\times\nu^k$.
\begin{lemma}
	\label{lem:bounded error}
	With probability at least $1-{\delta}/{4}$, we have for all $t\in[K]$, $h\in[H]$ and $\mu\in\Pi$,
	\begin{align}
	&(a)\quad\sum_{k=1}^{t-1}\E_{\pi^k}\bigg[\Big(f^{t,\mu}_h(s_h,a_h,b_h)-(\T^{\mu,\nu^t}_h f^{t,\mu}_{h+1})(s_h,a_h,b_h)\Big)^2\bigg]\leq\BO(\beta),\label{eq:BEE 1}\\
	&(b)\quad\sum_{k=1}^{t-1}\Big(f^{t,\mu}_h(s_h^k,a_h^k,b_h^k)-(\T^{\mu,\nu^t}_h f^{t,\mu}_{h+1})(s_h^k,a_h^k,b_h^k)\Big)^2\leq\BO(\beta).\label{eq:BEE 2}
	\end{align}	
\end{lemma}
\begin{proof}
	See Appendix~\ref{proof lemma bounded error}.
\end{proof}

Besides, using performance difference lemma we can easily bridge $\OVT(\mu^t)-V_1^{\pi^t}(s_1)$ with Bellman residuals, whose proof is deferred to Appendix~\ref{proof lemma performance}:
\begin{lemma}
\label{lem:performance}
For any $t\in[K]$, we have
\begin{align*}
\OVT(\mu^t)-V_1^{\pi^t}(s_1)=\sum_{h=1}^{H}\E_{\pi^t}\big[(f^{t,\mu^t}_h-\T^{\mu^t,\nu^t}_hf^{t,\mu^t}_{h+1})(s_h,a_h.b_h)\big].
\end{align*}
\end{lemma}

Therefore, from Lemma~\ref{lem:performance} we can obtain
\begin{align}
\label{eq:regret-bellman}
\sum_{t=1}^K\OVT(\mu^t)-V_1^{\pi^t}(s_1)=\sum_{h=1}^{H}\sum_{t=1}^K\E_{\pi^t}\big[(f^{t,\mu^t}_h-\T^{\mu^t,\nu^t}_hf^{t,\mu^t}_{h+1})(s_h,a_h,b_h)\big].
\end{align}

Notice that in (\ref{eq:regret-bellman}) we need to bound the Bellman residuals of $f^{t,\mu^t}_h$ weighted by policy $\pi^t$. However, in Lemma~\ref{lem:bounded error}, we can only bound the Bellman residuals weighted by $\pi^{1:t-1}$. Fortunately, we can utilize the inherent low BEE dimension to bridge these two values with the help of the following technical lemma:
\begin{lemma}[\citep{jin2021bellman}]
\label{lem:BEE}
Given a function class $\Phi$ defined on $\X$ with $\phi(x)\leq C$ for all $(\phi,x)\in\Phi\times\X$, and a family of probability measures $\Q$ over $X$. Suppose sequence $\{\phi_t\}_{t=1}^K\subset\Phi$ and $\{\rho_t\}_{t=1}^K\subset\Q$ satisfy that for all $t\in[K]$, $\sum_{k=1}^{t-1}(\E_{\rho_k}[\phi_t])^2\leq\beta$. Then for all $t\in[K]$ and $w>0$,
\begin{align*}
\sum_{k=1}^t|\E_{\rho_k}[\phi_k]|\leq\BO\Big(\sqrt{\DDE(\Phi,\Q,w)\beta t}+\min\{t,\DDE(\Phi,\Q,w)\}C+tw\Big).
\end{align*}  
\end{lemma}

Invoking Lemma~\ref{lem:BEE} with $\Q=\Q^1_h$, $\Phi=(I-\T^{\Pi,\Pi'}_h)\FC$ and $w=\sqrt{{1}/{K}}$, conditioning on the event (\ref{eq:BEE 1}) in Lemma~\ref{lem:bounded error} holds true, we have
\begin{align}
&\sum_{t=1}^K\E_{\pi^t}\big[(f^{t,\mu^t}_h-\T^{\mu^t,\nu^t}_hf^{t,\mu^t}_{h+1})(s_h,a_h.b_h)\big]\label{eq:berr 1}\\
&\qquad\leq\BO\bigg(\sqrt{\VB^2K\DBE\Big(\mathcal{F},\sqrt{{1}/{K}},\Pi,\Pi',\Q^1\Big)\log\left(\N_{\FC\cup\GC}(\VB/K)KH|\Pi|/\delta\right)}\bigg).\notag 
\end{align}

Similarly, invoking Lemma~\ref{lem:BEE} with $\Q=\Q^2_h$, $\Phi=(I-\T^{\Pi,\Pi'}_h)\FC$ and $w=\sqrt{{1}/{K}}$, conditioning on the event (\ref{eq:BEE 2}) in Lemma~\ref{lem:bounded error} holds true, we have with probability at least $1-\delta/4$,
\begin{align}
&\sum_{t=1}^K\E_{\pi^t}\big[(f^{t,\mu^t}_h-\T^{\mu^t,\nu^t}_hf^{t,\mu^t}_{h+1})(s_h,a_h.b_h)\big]\notag\\
&\qquad\leq\sum_{t=1}^{K}\Big(f^{t,\mu^t}_h(s_h,a_h,b_h)-(\T^{\mu,\nu^t}_h f^{t,\mu^t}_{h+1})(s^t_h,a^t_h,b^t_h)\Big)+\BO(\sqrt{K\log (K/\delta)})\notag\\
&\qquad\leq\BO\bigg(\sqrt{\VB^2K\DBE\Big(\mathcal{F},\sqrt{{1}/{K}},\Pi,\Pi',\Q^2\Big)\log(\N_{\FC\cup\GC}(\VB/K)KH|\Pi|/\delta)}\bigg),\label{eq:berr 2}
\end{align}
where the first inequality comes from standard martingale difference concentration. Therefore, combining (\ref{eq:berr 1}) and (\ref{eq:berr 2}),we have:
\begin{align*}
&\sum_{t=1}^K\E_{\pi^t}\big[(f^{t,\mu^t}_h-\T^{\mu^t,\nu^t}_hf^{t,\mu^t}_{h+1})(s_h,a_h.b_h)\big]\notag\\
&\qquad\leq\BO\bigg(\sqrt{\VB^2K\DBE\Big(\mathcal{F},\sqrt{{1}/{K}},\Pi,\Pi'\Big)\log(\N_{\FC\cup\GC}(\VB/K)KH|\Pi|/\delta)}\bigg).
\end{align*}

Substitute the above bounds into (\ref{eq:regret-bellman}) and we have:
\begin{align}
&\sum_{t=1}^K\OVT(\mu^t)-V_1^{\pi^t}(s_1)\label{eq:estimation}\\
&\qquad\leq\BO\bigg(H\VB\sqrt{K\DBE\Big(\mathcal{F},\sqrt{{1}/{K}},\Pi,\Pi'\Big)\log(\N_{\FC\cup\GC}(\VB/K)KH|\Pi|/\delta)}\bigg).\notag 
\end{align}

Thus, in Step 2, we establish an upper bound on term (4) in \eqref{eq:decompose}. It remains to bound term (2) and term (3), which is completed in the final step of the proof.

\subsection{Step 3: Bound the regret}
Now we only need to bound the online learning error. Notice that $p^t$ is updated using Hedge with reward $\OVT$. Since $0\leq\OVT\leq\VB$ and there are $|\Pi|$ policies, we have from the online learning literature \citep{hazan2016introduction} that
\begin{align}
\label{eq:online}
	\sum_{t=1}^{K}\OVT(\mu^*)-\sum_{t=1}^{K}\langle\OVT,p^t\rangle\leq\VB\sqrt{K\log|\Pi|}.
\end{align}

In addition, suppose $\Fil_{k}$ denotes the filtration induced by $\{\nu^1\}\cup(\cup_{i=1}^k\{\mu^{i},\D_{i},\nu^{i+1}\})$. Then we can observe that $\langle\OVT,p^t\rangle-\OVT(\mu^t)\in\Fil_{t}$. In addition, we have $\OVT\in\Fil_{t-1}$ since the estimation of $\OVT$ only utilizes $\D_{1:t-1}$, which implies
\begin{align*}
\E[\langle\OVT,p^t\rangle-\OVT(\mu^t)|\Fil_{t-1}]=0.
\end{align*}
Therefore $(3)$ is a martingale difference sequence and by Azuma-Hoeffding's inequality we have with probability at least $1-\delta/4$,
\begin{align}
\label{eq:martingale}
\sum_{t=1}^{K}\langle\OVT,p^t\rangle-\sum_{t=1}^{K}\OVT(\mu^t)\leq\BO(\VB\sqrt{K\log(1/\delta)})
\end{align} 

Substituting (\ref{eq:optimism}), (\ref{eq:estimation}), (\ref{eq:online}), and (\ref{eq:martingale}) into (\ref{eq:decompose}) concludes our proof for Theorem~\ref{thm:oblivious} in the oblivious setting.

Meanwhile, for the adaptive setting, we can simply repeat the above arguments. The only difference is that now $\nu^t$ can depend on $\D_{1:t-1}$ and thus we need to introduce a union bound over $\Pi'$ when proving Lemma~\ref{lem:optimism} and Lemma~\ref{lem:bounded error}. This will incur an additional $\log|\Pi'|$ in $\beta$ and thus also in the regret bound. This concludes our proof.

\section{Conclusion}

We study decentralized policy learning in general-sum Markov games. 
Specifically, we aim to establish a no-regret online learning algorithm for a single agent based on its local information, in the presence of  nonstationary and possibly adversarial opponents. 
Focusing on the policy revealing setting where the opponent's previous policies are revealed to the agent, we propose a novel algorithm that achieves sublinear regret in the context of general function approximation. Moreover, when all the agents adopt this algorithm, we prove that their mixture policy constitutes an approximate CCE of the Markov game. We further demonstrate the efficacy of the proposed algorithm by applying it to constrained and vector-valued MDPs, which can be formulated as zero-sum Markov games with a fictitious opponent. 
Finally, while we consider the policy revealing setting, 
establishing decentralized RL algorithm for Markov games under weaker information structures seems an important future direction. 


\clearpage 
\bibliographystyle{apalike}
\bibliography{ref.bib}

\begin{thebibliography}{}

\bibitem[Anderson, 2008]{anderson2008theory}
Anderson, T. (2008).
\newblock {\em The theory and practice of online learning}.
\newblock Athabasca University Press.

\bibitem[Azar et~al., 2017]{azar2017minimax}
Azar, M.~G., Osband, I., and Munos, R. (2017).
\newblock Minimax regret bounds for reinforcement learning.
\newblock In {\em Proceedings of the 34th International Conference on Machine
  Learning-Volume 70}, pages 263--272. JMLR. org.

\bibitem[Bai and Jin, 2020]{bai2020provable}
Bai, Y. and Jin, C. (2020).
\newblock Provable self-play algorithms for competitive reinforcement learning.
\newblock In {\em International conference on machine learning}, pages
  551--560. PMLR.

\bibitem[Bai et~al., 2020]{bai2020near}
Bai, Y., Jin, C., and Yu, T. (2020).
\newblock Near-optimal reinforcement learning with self-play.
\newblock {\em Advances in neural information processing systems},
  33:2159--2170.

\bibitem[Berner et~al., 2019]{berner2019dota}
Berner, C., Brockman, G., Chan, B., Cheung, V., D{\k{e}}biak, P., Dennison, C.,
  Farhi, D., Fischer, Q., Hashme, S., Hesse, C., et~al. (2019).
\newblock Dota 2 with large scale deep reinforcement learning.
\newblock {\em arXiv preprint arXiv:1912.06680}.

\bibitem[Blum and Monsour, 2007]{blum2007learning}
Blum, A. and Monsour, Y. (2007).
\newblock Learning, regret minimization, and equilibria.

\bibitem[Brafman and Tennenholtz, 2002]{brafman2002r}
Brafman, R.~I. and Tennenholtz, M. (2002).
\newblock R-max-a general polynomial time algorithm for near-optimal
  reinforcement learning.
\newblock {\em Journal of Machine Learning Research}, 3(Oct):213--231.

\bibitem[Brambilla et~al., 2013]{brambilla2013swarm}
Brambilla, M., Ferrante, E., Birattari, M., and Dorigo, M. (2013).
\newblock Swarm robotics: a review from the swarm engineering perspective.
\newblock {\em Swarm Intelligence}, 7(1):1--41.

\bibitem[Canese et~al., 2021]{canese2021multi}
Canese, L., Cardarilli, G.~C., Di~Nunzio, L., Fazzolari, R., Giardino, D., Re,
  M., and Span{\`o}, S. (2021).
\newblock Multi-agent reinforcement learning: A review of challenges and
  applications.
\newblock {\em Applied Sciences}, 11(11):4948.

\bibitem[Daskalakis et~al., 2011]{daskalakis2011near}
Daskalakis, C., Deckelbaum, A., and Kim, A. (2011).
\newblock Near-optimal no-regret algorithms for zero-sum games.
\newblock In {\em Proceedings of the twenty-second annual ACM-SIAM symposium on
  Discrete Algorithms}, pages 235--254. SIAM.

\bibitem[Daskalakis et~al., 2020]{daskalakis2020independent}
Daskalakis, C., Foster, D.~J., and Golowich, N. (2020).
\newblock Independent policy gradient methods for competitive reinforcement
  learning.
\newblock {\em Advances in neural information processing systems},
  33:5527--5540.

\bibitem[Ding et~al., 2022]{ding2022independent}
Ding, D., Wei, C.-Y., Zhang, K., and Jovanovi{\'c}, M.~R. (2022).
\newblock Independent policy gradient for large-scale markov potential games:
  Sharper rates, function approximation, and game-agnostic convergence.
\newblock {\em arXiv preprint arXiv:2202.04129}.

\bibitem[Ding et~al., 2021]{ding2021provably}
Ding, D., Wei, X., Yang, Z., Wang, Z., and Jovanovic, M. (2021).
\newblock Provably efficient safe exploration via primal-dual policy
  optimization.
\newblock In {\em International Conference on Artificial Intelligence and
  Statistics}, pages 3304--3312. PMLR.

\bibitem[Du et~al., 2021]{du2021bilinear}
Du, S.~S., Kakade, S.~M., Lee, J.~D., Lovett, S., Mahajan, G., Sun, W., and
  Wang, R. (2021).
\newblock Bilinear classes: A structural framework for provable generalization
  in rl.

\bibitem[Efroni et~al., 2020]{efroni2020exploration}
Efroni, Y., Mannor, S., and Pirotta, M. (2020).
\newblock Exploration-exploitation in constrained {MDPs}.
\newblock {\em arXiv preprint arXiv:2003.02189}.

\bibitem[Freund and Schapire, 1997]{freund1997decision}
Freund, Y. and Schapire, R.~E. (1997).
\newblock A decision-theoretic generalization of on-line learning and an
  application to boosting.
\newblock {\em Journal of computer and system sciences}, 55(1):119--139.

\bibitem[Fudenberg and Tirole, 1991]{fudenberg1991game}
Fudenberg, D. and Tirole, J. (1991).
\newblock {\em Game theory}.
\newblock MIT press.

\bibitem[Gronauer and Diepold, 2022]{gronauer2022multi}
Gronauer, S. and Diepold, K. (2022).
\newblock Multi-agent deep reinforcement learning: a survey.
\newblock {\em Artificial Intelligence Review}, 55(2):895--943.

\bibitem[Gupta et~al., 2017]{gupta2017cooperative}
Gupta, J.~K., Egorov, M., and Kochenderfer, M. (2017).
\newblock Cooperative multi-agent control using deep reinforcement learning.
\newblock In {\em International conference on autonomous agents and multiagent
  systems}, pages 66--83. Springer.

\bibitem[Hazan et~al., 2016]{hazan2016introduction}
Hazan, E. et~al. (2016).
\newblock Introduction to online convex optimization.
\newblock {\em Foundations and Trends{\textregistered} in Optimization},
  2(3-4):157--325.

\bibitem[Hernandez-Leal et~al., 2018]{hernandez2018multiagent}
Hernandez-Leal, P., Kartal, B., and Taylor, M.~E. (2018).
\newblock Is multiagent deep reinforcement learning the answer or the question?
  a brief survey.
\newblock {\em learning}, 21:22.

\bibitem[Hernandez-Leal et~al., 2019]{hernandez2019survey}
Hernandez-Leal, P., Kartal, B., and Taylor, M.~E. (2019).
\newblock A survey and critique of multiagent deep reinforcement learning.
\newblock {\em Autonomous Agents and Multi-Agent Systems}, 33(6):750--797.

\bibitem[Huang et~al., 2021]{huang2021towards}
Huang, B., Lee, J.~D., Wang, Z., and Yang, Z. (2021).
\newblock Towards general function approximation in zero-sum markov games.
\newblock {\em arXiv preprint arXiv:2107.14702}.

\bibitem[Jiang et~al., 2017]{jiang2017contextual}
Jiang, N., Krishnamurthy, A., Agarwal, A., Langford, J., and Schapire, R.~E.
  (2017).
\newblock Contextual decision processes with low bellman rank are
  pac-learnable.
\newblock In {\em International Conference on Machine Learning}, volume~70,
  pages 1704--1713. PMLR.

\bibitem[Jin et~al., 2021a]{jin2021bellman}
Jin, C., Liu, Q., and Miryoosefi, S. (2021a).
\newblock Bellman eluder dimension: New rich classes of rl problems, and
  sample-efficient algorithms.
\newblock {\em Advances in Neural Information Processing Systems}, 34.

\bibitem[Jin et~al., 2021b]{jin2021v}
Jin, C., Liu, Q., Wang, Y., and Yu, T. (2021b).
\newblock V-learning--a simple, efficient, decentralized algorithm for
  multiagent rl.
\newblock {\em arXiv preprint arXiv:2110.14555}.

\bibitem[Jin et~al., 2021c]{jin2021power}
Jin, C., Liu, Q., and Yu, T. (2021c).
\newblock The power of exploiter: Provable multi-agent rl in large state
  spaces.
\newblock {\em arXiv preprint arXiv:2106.03352}.

\bibitem[Jin et~al., 2020]{jin2020provably}
Jin, C., Yang, Z., Wang, Z., and Jordan, M.~I. (2020).
\newblock Provably efficient reinforcement learning with linear function
  approximation.
\newblock In {\em Conference on Learning Theory}, pages 2137--2143. PMLR.

\bibitem[Konda and Tsitsiklis, 1999]{konda1999actor}
Konda, V. and Tsitsiklis, J. (1999).
\newblock Actor-critic algorithms.
\newblock {\em Advances in neural information processing systems}, 12.

\bibitem[Lattimore and Szepesv{\'a}ri, 2020]{lattimore2020bandit}
Lattimore, T. and Szepesv{\'a}ri, C. (2020).
\newblock {\em Bandit algorithms}.
\newblock Cambridge University Press.

\bibitem[LeCun et~al., 2015]{lecun2015deep}
LeCun, Y., Bengio, Y., and Hinton, G. (2015).
\newblock Deep learning.
\newblock {\em Nature}, 521(7553):436--444.

\bibitem[Leonardos et~al., 2021]{leonardos2021global}
Leonardos, S., Overman, W., Panageas, I., and Piliouras, G. (2021).
\newblock Global convergence of multi-agent policy gradient in markov potential
  games.
\newblock {\em arXiv preprint arXiv:2106.01969}.

\bibitem[Liu et~al., 2022]{liu2022learning}
Liu, Q., Wang, Y., and Jin, C. (2022).
\newblock Learning markov games with adversarial opponents: Efficient
  algorithms and fundamental limits.
\newblock {\em arXiv preprint arXiv:2203.06803}.

\bibitem[Liu et~al., 2021]{liu2021learning}
Liu, T., Zhou, R., Kalathil, D., Kumar, P., and Tian, C. (2021).
\newblock Learning policies with zero or bounded constraint violation for
  constrained mdps.
\newblock {\em Advances in Neural Information Processing Systems}, 34.

\bibitem[Mao et~al., 2021]{mao2021on}
Mao, W., Yang, L.~F., Zhang, K., and Başar, T. (2021).
\newblock On improving model-free algorithms for decentralized multi-agent
  reinforcement learning.

\bibitem[Miryoosefi et~al., 2019]{miryoosefi2019reinforcement}
Miryoosefi, S., Brantley, K., Daume~III, H., Dudik, M., and Schapire, R.~E.
  (2019).
\newblock Reinforcement learning with convex constraints.
\newblock {\em Advances in Neural Information Processing Systems}, 32.

\bibitem[Miryoosefi and Jin, 2021]{miryoosefi2021simple}
Miryoosefi, S. and Jin, C. (2021).
\newblock A simple reward-free approach to constrained reinforcement learning.
\newblock {\em arXiv preprint arXiv:2107.05216}.

\bibitem[Mnih et~al., 2013]{mnih2013playing}
Mnih, V., Kavukcuoglu, K., Silver, D., Graves, A., Antonoglou, I., Wierstra,
  D., and Riedmiller, M. (2013).
\newblock Playing atari with deep reinforcement learning.
\newblock {\em arXiv preprint arXiv:1312.5602}.

\bibitem[Paternain et~al., 2019]{paternain2019constrained}
Paternain, S., Chamon, L., Calvo-Fullana, M., and Ribeiro, A. (2019).
\newblock Constrained reinforcement learning has zero duality gap.
\newblock {\em Advances in Neural Information Processing Systems}, 32.

\bibitem[Puterman, 1994]{puterman1994markov}
Puterman, M.~L. (1994).
\newblock Markov decision processes: Discrete stochastic dynamic programming.

\bibitem[Rashid et~al., 2018]{rashid2018qmix}
Rashid, T., Samvelyan, M., Schroeder, C., Farquhar, G., Foerster, J., and
  Whiteson, S. (2018).
\newblock Qmix: Monotonic value function factorisation for deep multi-agent
  reinforcement learning.
\newblock In {\em International Conference on Machine Learning}, pages
  4295--4304. PMLR.

\bibitem[Russo and Van~Roy, 2013]{russo2013eluder}
Russo, D. and Van~Roy, B. (2013).
\newblock Eluder dimension and the sample complexity of optimistic exploration.
\newblock {\em Advances in Neural Information Processing Systems}, 26.

\bibitem[Sayin et~al., 2021]{sayin2021decentralized}
Sayin, M., Zhang, K., Leslie, D., Basar, T., and Ozdaglar, A. (2021).
\newblock Decentralized q-learning in zero-sum markov games.
\newblock {\em Advances in Neural Information Processing Systems}, 34.

\bibitem[Shalev-Shwartz et~al., 2016]{shalev2016safe}
Shalev-Shwartz, S., Shammah, S., and Shashua, A. (2016).
\newblock Safe, multi-agent, reinforcement learning for autonomous driving.
\newblock {\em arXiv preprint arXiv:1610.03295}.

\bibitem[Silver et~al., 2016]{silver2016mastering}
Silver, D., Huang, A., Maddison, C.~J., Guez, A., Sifre, L., Van Den~Driessche,
  G., Schrittwieser, J., Antonoglou, I., Panneershelvam, V., Lanctot, M.,
  et~al. (2016).
\newblock Mastering the game of {G}o with deep neural networks and tree search.
\newblock {\em nature}, 529(7587):484--489.

\bibitem[Srinivas et~al., 2009]{srinivas2009gaussian}
Srinivas, N., Krause, A., Kakade, S.~M., and Seeger, M. (2009).
\newblock Gaussian process optimization in the bandit setting: No regret and
  experimental design.
\newblock {\em arXiv preprint arXiv:0912.3995}.

\bibitem[Tian et~al., 2021]{tian2021online}
Tian, Y., Wang, Y., Yu, T., and Sra, S. (2021).
\newblock Online learning in unknown markov games.
\newblock In {\em International conference on machine learning}, pages
  10279--10288. PMLR.

\bibitem[Vinyals et~al., 2019]{vinyals2019grandmaster}
Vinyals, O., Babuschkin, I., Czarnecki, W.~M., Mathieu, M., Dudzik, A., Chung,
  J., Choi, D.~H., Powell, R., Ewalds, T., Georgiev, P., et~al. (2019).
\newblock Grandmaster level in starcraft ii using multi-agent reinforcement
  learning.
\newblock {\em Nature}, 575(7782):350--354.

\bibitem[Wainwright, 2019]{wainwright2019high}
Wainwright, M. (2019).
\newblock {\em High-Dimensional Statistics: A Non-Asymptotic Viewpoint},
  volume~48 of {\em Cambridge Series in Statistical and Probabilistic
  Mathematics}.
\newblock Cambridge University Press.

\bibitem[Wang et~al., 2019]{wang2019optimism}
Wang, Y., Wang, R., Du, S.~S., and Krishnamurthy, A. (2019).
\newblock Optimism in reinforcement learning with generalized linear function
  approximation.
\newblock {\em arXiv preprint arXiv:1912.04136}.

\bibitem[Wei et~al., 2017]{wei2017online}
Wei, C.-Y., Hong, Y.-T., and Lu, C.-J. (2017).
\newblock Online reinforcement learning in stochastic games.
\newblock {\em Advances in Neural Information Processing Systems}, 30.

\bibitem[Wei et~al., 2021a]{wei2021last}
Wei, C.-Y., Lee, C.-W., Zhang, M., and Luo, H. (2021a).
\newblock Last-iterate convergence of decentralized optimistic gradient
  descent/ascent in infinite-horizon competitive markov games.
\newblock In {\em Conference on Learning Theory}, pages 4259--4299. PMLR.

\bibitem[Wei et~al., 2021b]{wei2021provably}
Wei, H., Liu, X., and Ying, L. (2021b).
\newblock A provably-efficient model-free algorithm for constrained markov
  decision processes.
\newblock {\em arXiv preprint arXiv:2106.01577}.

\bibitem[Xie et~al., 2020]{xie2020learning}
Xie, Q., Chen, Y., Wang, Z., and Yang, Z. (2020).
\newblock Learning zero-sum simultaneous-move markov games using function
  approximation and correlated equilibrium.
\newblock In {\em Conference on learning theory}, pages 3674--3682. PMLR.

\bibitem[Xie et~al., 2021]{xie2021bellman}
Xie, T., Cheng, C.-A., Jiang, N., Mineiro, P., and Agarwal, A. (2021).
\newblock Bellman-consistent pessimism for offline reinforcement learning.
\newblock {\em arXiv preprint arXiv:2106.06926}.

\bibitem[Yu et~al., 2021]{yu2021provably}
Yu, T., Tian, Y., Zhang, J., and Sra, S. (2021).
\newblock Provably efficient algorithms for multi-objective competitive rl.
\newblock In {\em International Conference on Machine Learning}, pages
  12167--12176. PMLR.

\bibitem[Zhang et~al., 2021a]{zhang2021multi}
Zhang, K., Yang, Z., and Ba{\c{s}}ar, T. (2021a).
\newblock Multi-agent reinforcement learning: A selective overview of theories
  and algorithms.
\newblock {\em Handbook of Reinforcement Learning and Control}, pages 321--384.

\bibitem[Zhang et~al., 2021b]{zhang2021gradient}
Zhang, R., Ren, Z., and Li, N. (2021b).
\newblock Gradient play in multi-agent markov stochastic games: Stationary
  points and convergence.
\newblock {\em arXiv e-prints}, pages arXiv--2106.

\end{thebibliography}

\clearpage 
\appendix
\section{Proofs of Proposition~\ref{prop:BEE}}
\label{proof prop BEE}
From the completeness assumption, we know that there exists $g_h\in\FC_h$ such that $g_h=\T^{\mu,\nu}_hf_{h+1}$, 
which implies that
\begin{align*}
f_h-\T^{\mu,\nu}_hf_{h+1}\in\FC_h-\FC_h,\forall f\in\FC, \mu\in\Pi,\nu\in\Pi'.
\end{align*}

In other words, $(I-\T^{\Pi,\Pi'}_h)\FC\subseteq\FC_h-\FC_h$. Therefore, from the definition of $\DBE(\FC,\eps,\break\Pi,\Pi')$ we have
\begin{align*}
\DBE(\FC,\eps,\Pi,\Pi')&\leq\DBE(\FC,\eps,\Pi,\Pi',\Q^2)=\max_{h\in[H]}\DDE((I-\T^{\Pi,\Pi'}_h)\FC,\Q^2_h,\eps)\\
&\leq\max_{h\in[H]}\DDE((\FC_h-\FC_h),\Q^2_h,\eps)=\max_{h\in[H]}\DE(\FC_h,\eps),
\end{align*}
where the last step comes from the definition of $\DE$ and $\Q^2_h$ is the dirac distribution family. This concludes our proof.
\section{Examples for Realizability, Generalized Completeness and Covering Number}
\label{sec:realize}
In this section we illustrate practical examples where realizability and generalized completeness hold while the covering number is upper bounded at the same time. More specifically, we will consider tabular MGs, linear MGs and kernel MGs.
\subsection{Tabular MGs}
For tabular MGs, we let $\FC_h=\{f|f:\SC\times\A\times\B\mapsto[0,\VB]\}$ and $\GC_h=\FC_h$ for all $h\in[H]$. Then it is obvious that $Q^{\mu\times\nu}_h\in\FC_h$ and $\T^{\mu,\nu}f_{h+1}\in\GC_h$ for any $f\in\FC,h\in[H],\mu,\nu$, which implies that realizability and generalized completeness are satisfied. In addition, notice that in this case we have
\begin{align*}
\log\N_{\FC_h}(\eps)=\log\N_{\GC_h}(\eps)\leq|\SC||\A||\B|\log({\VB}/{\eps}).
\end{align*}
This suggests that the size of $\FC$ and $\GC$ is also not too large.

\subsection{Linear MGs}
In this subsection we consider linear MGs. Here we generalize the definition of linear MDPs in classic MDPs \citep{jin2020provably} to Markov games:
\begin{definition}[Linear MGs]
We say an MG is linear of dimension $d$ if for each $h\in[H]$, there exists a feature mapping $\phi_h:\SC\times\A\times\B\mapsto\R^d$ and $d$ unknown signed measures $\psi_h=(\psi_h^{(1)},\cdots,\psi_h^{(d)})$ over $\SC$ and an unknown vector $\theta_h\in\R^d$ such that $P_h(\cdot|s,a,b)=\phi_{h}(s,a,b)^{\top}\psi_h(\cdot)$ and $r_{h}(s,a,b)=\phi_h(s,a,b)^{\top}\theta_h$ for all $(s,a,b)\in\SC\times\A\times\B$.
\end{definition}

Without loss of generality, we assume $\Vert\phi_h(s,a,b)\Vert\leq1$ for all $s\in\SC,a\in\A,b\in\B$ and $\Vert\psi_h(\SC)\Vert\leq\sqrt{d},\theta_h\leq\sqrt{d}$ for all $h$. Let $\FC_h=\GC_h=\{\phi_h(\cdot)^{\top}w|w\in\R^d,\Vert w\Vert\leq (H-h+1)\sqrt{d},0\leq\phi_h(\cdot)^{\top}w\leq H-h+1\}$.

\paragraph{Realizability.}
We have for any $\mu,\nu$,
\begin{align*}
Q^{\mu\times\nu}_h(s,a,b)&=r_h(s,a,b)+\E_{s'\sim P(\cdot|s,a,b)}[V^{\mu\times\nu}_{h+1}(s')]\\
&=\langle \phi_h(s,a,b),\theta_h\rangle + \bigg\langle \phi_h(s,a,b),\int_{\SC}V^{\mu\times\nu}_{h+1}(s')d\psi_h(s')\bigg\rangle\\
&=\bigg\langle \phi_h(s,a,b), \theta_h+\int_{\SC}V^{\mu\times\nu}_{h+1}(s')d\psi_h(s')\bigg\rangle\\
&=\langle \phi_h(s,a,b), w^{\mu\times\nu}_h\rangle,
\end{align*}
where $w^{\mu\times\nu}_h=\theta_h+\int_{\SC}V^{\mu\times\nu}_{h+1}(s')d\psi_h(s')$ and thus $\Vert w^{\mu\times\nu}_h\Vert\leq (H-h+1)\sqrt{d}$. Therefore, $Q^{\mu\times\nu}_h\in\FC_h$, which means that realizability holds.

\paragraph{Generalized completeness.} For any $f_{h+1}\in\FC_{h+1}$, we have
\begin{align*}
\T^{\mu,\nu}f_{h+1}(s,a,b)&=r_h(s,a,b)+\E_{s'\sim P(\cdot|s,a,b)}[f_{h+1}(s',\mu,\nu)]\\
&=\bigg\langle \phi_h(s,a,b), \theta_h+\int_{\SC}f_{h+1}(s',\mu,\nu)d\psi_h(s')\bigg\rangle.
\end{align*}
Since $\Vert f_{h+1}\Vert_{\infty}\leq H-h$, we have $\Vert \theta_h+\int_{\SC}f_{h+1}(s',\mu,\nu)d\psi_h(s')\Vert\leq (H-h+1)\sqrt{d}$, which indicates $\T^{\mu,\nu}f_{h+1}\in\GC_h$ and thus generalized completeness is satisfied.

\paragraph{Covering number.} First notice that from the literature \citep{wainwright2019high}, the covering number of a $l_2$-norm ball can be bounded as $\log\N_{\BL((H-h+1)\sqrt{d})}(\eps)\leq d\log({3H\sqrt{d}}/{\eps})$. Therefore, there exists $\WC\subset\BL((H-h+1)\sqrt{d})$ where $\log|\WC|\leq d\log({3H\sqrt{d}}/{\eps})$ such that for any $w\in\BL((H-h+1)\sqrt{d})$, there exists $w'\in\WC$ satisfying $\Vert w'-w\Vert\leq \eps$. Now let $\FC'_h=\{\phi_h(\cdot)^{\top}w|w\in\WC\}$. For any $f_h\in\FC_h$, suppose $f_h(\cdot)=\phi_h(\cdot)^{\top} w_{f_h}$. Then we know there exists $f'_h(\cdot)=\phi_h(\cdot)^{\top} w'_{f_h}\in\FC'_h$ where $\Vert w'_{f_h}-w_{f_h}\Vert\leq \eps$, which implies
\begin{align*}
|f_h(s,a,b)-f'_h(s,a,b)|\leq\Vert\phi_h(s,a,b)\Vert\Vert w'_{f_h}-w_{f_h}\Vert\leq\eps.
\end{align*}
Therefore $\log\N_{\FC_h}(\eps)\leq\log|\FC'_h|=\log|\WC|\leq d\log({3H\sqrt{d}}/{\eps})$.

\subsection{Kernel MGs}
In this subsection we show that kernel MGs also satisfy realizability and generalized completeness naturally. In addition, when a kernel MG has a bounded effective dimension, its covering number will also be bounded. First we generalize the definition of kernel MDPs \cite{jin2021bellman} to MGs as follows. 
\begin{definition}[Kernel MGs]
\label{def:kernel MG}
In a kernel MDP, for each step $h\in[H]$, there exist feature mapping $\phi_h:\SC\times\A\times\B\mapsto\HC$ and $\psi_h:\SC\mapsto\HC$ where $\HC$ is a separable Hilbert space such that $P_h(s'|s,a,b)=\langle\phi_h(s,a,b),\psi_h(s')\rangle_{\HC}$ for all $s\in\SC,a\in\A,b\in\B,s'\in\SC$. Besides, the reward function os linear in $\phi$, i.e., $r_h(s,a,b)=\langle \phi_h(s,a,b),\theta_h\rangle_{\HC} $ for some $\theta_h\in\HC$. Moreover, a kernel MG satisfies the following regularization conditions:
\begin{itemize}
\item $\Vert\theta_h\Vert_{\HC}\leq1,\Vert\phi_h(s,a,b)\Vert_{\HC}\leq1$, for all $ s\in\SC,a\in\A,b\in\B,h\in[H]$.
\item $\Vert\sum_{s\in\SC}V(s)\psi_h(s)\Vert_{\HC}\leq1$, for all function $V:\SC\mapsto[0,1], h\in[H]$.
\end{itemize}
\end{definition}
\begin{remark}
It can be observed that tabular and linear MGs are special cases of kernel MGs. Therefore, the following discussion applies to tabular and linear MGs as well.
\end{remark}

Then we let $\FC_h=\GC_h=\{\phi_h(\cdot)^{\top}w|w\in\B_{\HC}(H-h+1)\}$ where $\B_{\HC}(r)$ is a ball with radius $r$ in $\HC$. Following the same arguments in linear MGs, we can validate that realizability and generalized completeness are satisfied in kernel MGs. 

\paragraph{Covering number.} Before bounding the covering number of $\FC_h$, we need introduce a new measure to evaluate the complexity of a Hilbert space since $\HC$ might be infinite dimensional. Here we use the effective dimension \citep{du2021bilinear,jin2021bellman}, which is defined as follows:
\begin{definition}[$\eps$-effective dimension of a set]
\label{def:effect}
The $\eps$-effective dimension of a set $\X$ is the minimum integer $\deff(\X,\eps)=n$ such that
\begin{align*}
	\sup_{x_1,\cdots,x_n\in\X}\frac{1}{n}\log\det\bigg(I+\frac{1}{\eps^2}\sum_{i=1}^nx_ix_i^{\top}\bigg)\leq e^{-1}.
\end{align*}
\end{definition}
\begin{remark}
When $\X$ is finite dimensional, suppose its dimension is $d$. Then its effective dimension can be upper bounded by $\BO\big(d\log\big(1+{R^2}/{\eps}\big)\big)$ where $R$ is the norm bound of $\X$ \citep{du2021bilinear}. In addition, even when $\X$ is infinite dimensional, if the eigenspectrum of the covariance matrices concentrates in a low-dimension
subspace, the effective dimension of $\X$ can still be small \citep{srinivas2009gaussian}.
\end{remark}

We call a kernel MG is of effective dimension $d(\eps)$ if $\deff(\X_h,\eps)\leq d(\eps)$ for all $h$ and $\eps$ where $\X_h=\{\phi_h(s,a,b):(s,a,b)\in\SC\times\A\times\B\}$. Then the following proposition shows that the covering number of $\FC_h$ is upper bounded by the effective dimension of the kernel MG:
\begin{proposition}
\label{prop:covering}
If the kernel MG has effective dimension $d(\eps)$, then
\begin{align*}
\log\N_{\FC_h}(\eps)\leq\BO\big(d({\eps}/{2H})\log(1+{Hd({\eps}/{2H})}/{\eps})\big).
\end{align*}
\end{proposition}
\begin{proof}
Suppose $\DE(\FC_h,\eps)=n$. Then by the definition of Eluder dimension, there exists a sequence $\{\phi_i\}_{i=1}^n$ such that for any $w_1,w_2\in\BL_{\HC}(H-h+1),\phi\in\X_h$, if $\sum_{i=1}^n(\langle \phi_i, w_1-w_2\rangle)^2\leq\eps^2$, then $|\langle \phi, w_1-w_2\rangle|\leq\eps$. Therefore, the covering number of kernel MGs can be reduced to covering the projection of $\BL_{\HC}(H-h+1)$ onto the space spanned by $\{\phi_i\}_{i=1}^n$, whose dimension is at most $n$. From the literature \citep{wainwright2019high}, the covering number of such space is $\BO\left(n\log\left(1+{nH}/{\eps}\right)\right)$, which implies
\begin{align*}
\log\N_{\FC_h}(\eps)\leq\BO\big(n\log(1+{nH}/{\eps})\big).
\end{align*}

Finally, by the proof of Proposition~\ref{prop:BEE kernel}, we know $n\leq d({\eps}/{2H})$, which concludes the proof.

\end{proof}
\section{Examples for BEE Dimension}
\label{sec:example}
In this section we will show that kernel MGs (including tabular MGs and linear MGs) and generalized linear complete models have low BEE dimensions.

\subsection{Kernel MGs}
Consider the kernel MG defined in Definition~\ref{def:kernel MG} and $\FC_h=\{\phi_h(\cdot)^{\top}w|w\in\B_{\HC}(H-h+1)\}$, then we have the following proposition showing that the BEE dimension of a kernel MG is upper bounded by its effective dimension (Definition~\ref{def:effect}):
\begin{proposition}
\label{prop:BEE kernel}
If the kernel MG has effective dimension $d(\eps)$, then for any policy classes $\Pi$ and $\Pi'$, we have $\dbee(\FC,\eps,\Pi,\Pi')\leq d({\eps}/{2H})$.
\end{proposition}
\begin{proof}
First in Appendix~\ref{sec:realize} we have showed that $\FC$ satisfies completeness. By Proposition~\ref{prop:BEE}, we have $\dbee(\FC,\eps,\Pi,\Pi')\leq\max_{h\in[H]}\DE(\FC_h,\eps)$. Therefore we only need to bound $\DE(\FC_h,\eps)$ for each $h\in[H]$. Suppose $\DE(\FC_h,\eps)=k>d({\eps}/{2H})$. Then by the definition of Eluder dimension, there exists a sequence $\phi_1,\cdots,\phi_k$ and $\{w_{1,i}\}_{i=1}^k,\{w_{2,i}\}_{i=1}^k$ where $\phi_i\in\X_h=\{\phi_h(s,a,b):(s,a,b)\in\SC\times\A\times\B\}, w_{1,i},w_{2,i}\in\B_{\HC}(H-h+1)$ for all $i$ such that for any $t\in[k]$:
\begin{align}
&\sum_{i=1}^{t-1}(\langle \phi_i,w_{1,t}-w_{2,t}\rangle)^2\leq(\eps')^2,\label{eq:eg 1}\\
&|\langle\phi_t,w_{1,t}-w_{2,t}\rangle|\geq\eps'\label{eq:eg 2},
\end{align}
where $\eps'\geq\eps$. Let $\Sigma_t$ denote $\sum_{i=1}^{t-1}\phi_i\phi_i^{\top}+\frac{\eps^2}{4H^2}\cdot I$. Then we have for any $t\in[k]$
\begin{align*}
\Vert w_{1,t}-w_{2,t}\Vert_{\Sigma_t}^2\leq(\eps')^2+\eps^2.
\end{align*}

On the other hand, by Cauchy-Schwartz inequality we know
\begin{align*}
\Vert \phi_t\Vert_{\Sigma_t^{-1}}\Vert w_{1,t}-w_{2,t}\Vert_{\Sigma_t}\geq|\langle\phi_t,w_{1,t}-w_{2,t}\rangle|\geq\eps'.
\end{align*}
This implies for all $t\in[k]$
\begin{align*}
\Vert \phi_t\Vert_{\Sigma_t^{-1}}\geq\frac{\eps'}{\sqrt{\eps^2+(\eps')^2}}\geq\frac{1}{\sqrt{2}}.
\end{align*}

Therefore, applying elliptical potential lemma (e.g., Lemma~5.6 and Lemma~F.3 in \cite{du2021bilinear}), we have for any $t\in[k]$
\begin{align*}
\log\det\bigg(I+\frac{4H^2}{\eps^2}\sum_{i=1}^{t}\phi_i\phi_i^{\top}\bigg)=\sum_{i=1}^t\log(1+\Vert \phi_i\Vert_{\Sigma_i^{-1}}^2)\geq t\cdot\log\frac{3}{2}.
\end{align*}

However, by the definition of effective dimension, we know when $n=\deff(\X_h,\frac{\eps}{2H})$,
\begin{align*}
	\sup_{\phi_1,\cdots,\phi_n}\log\det\bigg(I+\frac{4H^2}{\eps^2}\sum_{i=1}^n\phi_i\phi_i^{\top}\bigg)\leq ne^{-1}.
\end{align*}

This is a contradiction since $n\leq d({\eps}/{2H})<k$ and $\log\frac{3}{2}>e^{-1}$. Therefore we have $\DE(\FC_h,\eps)\leq d({\eps}/{2H})$ for all $h\in[H]$, which implies
\begin{align*}
\dbee(\FC,\eps,\Pi,\Pi')\leq d({\eps}/{2H}).
\end{align*}
This concludes our proof.
\end{proof}

\paragraph{Tabular MGs.} Tabular MGs are a special case of kernel MGs where the feature vectors are $|\SC||\A||\B|$-dimensional one-hot vectors. From the standard elliptical potential lemma, we know $d(\eps)=\TO(|\SC||\A||\B|)$ for tabular MDPs, suggesting their BEE dimension is also upper bounded $\TO(|\SC||\A||\B|)$.

\paragraph{Linear MGs.} When the feature vectors are $d$-dimensional, we can recover linear MGs. Similarly, by the standard elliptical potential lemma, we have the BEE dimension of linear MGs is upper bounded $\TO(d)$.

\subsection{Generalized Linear Complete Models}
An important variant of linear MDPs is the generalized linear complete models proposed by \cite{wang2019optimism}. Here we also generalize it into Markov games:
\begin{definition}[Generalized linear complete models]
In $d$-dimensional generalized linear complete models, for each step $h\in[H]$, there exists a feature mapping $\phi_{h}:\SC\times\A\times\B\mapsto\R^d$ and a link function $\sigma$ such that:
\begin{itemize}
	\item for the generalized linear function class $\FC_h=\{\sigma(\phi_h(\cdot)^{\top}w)|w\in\WC\}$ where $\WC\subset\R^d$, realizability and completeness are both satisfied;
	\item the link function is strictly monotone, i.e., there exist $0<c_1<c_2<\infty$ such that $\sigma'\in[c_1,c_2]$.
	\item $\phi_h,w$ satisfy the regularization conditions: $\Vert\phi_h(s,a,b)\Vert\leq R, \Vert w\Vert\leq R$ for all $s,a,b,h$ where $R>0$ is a constant. 
\end{itemize}
\end{definition} 

When the link function is $\sigma(x)=x$, the generalized linear complete models reduce to the linear complete models, which contain instances such as linear MGs and LQRs. The following proposition shows that generalized linear complete models also have low BEE dimensions:
\begin{proposition}
	\label{prop:BEE generalized}
	If a generalized linear complete model has dimension $d$, then for any policy classes $\Pi$ and $\Pi'$, its BEE dimension can be bounded as follows:
	\begin{align*}
		\dbee(\FC,\eps,\Pi,\Pi')\leq \TO(d{c_2^2}/{c_1^2}).
	\end{align*}
\end{proposition}
\begin{proof}
The proof is similar to Proposition~\ref{prop:BEE kernel}, except (\ref{eq:eg 1}) and (\ref{eq:eg 2}) become
\begin{align*}
	&\sum_{i=1}^{t-1}c_1^2(\langle \phi_i,w_{1,t}-w_{2,t}\rangle)^2\leq\sum_{i=1}^{t-1}(\sigma(\phi_i^{\top} w_{1,t})-\sigma(\phi_i^{\top} w_{2,t}))^2\leq(\eps')^2,\\\
	&c_2|\langle\phi_t,w_{1,t}-w_{2,t}\rangle|\geq|\sigma(\phi_t^{\top} w_{1,t})-\sigma(\phi_t^{\top} w_{2,t})|\geq\eps'.
\end{align*}
Then repeat the arguments in the proof of Proposition~\ref{prop:BEE kernel}, we have $\DE(\FC_h,\eps)\leq\TO(d{c_2^2}/{c_1^2})$ for all $h\in[H]$. Since $\FC$ satisfies completeness, we can use Proposition~\ref{prop:BEE} and obtain 
\begin{align*}
	\dbee(\FC,\eps,\Pi,\Pi')\leq \TO(d{c_2^2}/{c_1^2}).
\end{align*}

\end{proof}
\section{Proofs of Lemmas in Section~\ref{sec:proof sketch}}
\subsection{Proof of Lemma~\ref{lem:optimism}}
\label{proof lemma optimism}
Let $\VC$ be a $\rho$-cover of $\GC$ with respect to $\Vert\cdot\Vert_{\infty}$. Consider an arbitrary fixed tuple $(\mu,t,h,g)\in\Pi\times[K]\times[H]\times\GC$. Define $W_{t,k}(h,g,\mu)$ as follows:
\begin{align*}
W_{t,k}(h,g,\mu):=&(g_h(s^k_h,a^k_h,b_h^k)-r_h^k-Q^{\mu,\nu^t}_{h+1}(s^k_{h+1},\mu,\nu^t))^2\notag\\
&-(Q^{\mu,\nu^t}_{h}(s_h^k,a_h^k,b_h^k)-r_h^k-Q^{\mu,\nu^t}_{h+1}(s^k_{h+1},\mu,\nu^t))^2,
\end{align*}
and $\Fil_{k,h}$ be the filtration induced by $\{\nu^1,\cdots,\nu^K\}\cup\{s^i_1,a^i_1,b^i_1,r^i_1,\cdots,s^i_{H+1}\}_{i=1}^{k-1}\cup\break\{s^k_1,a^k_1,b^k_1,r^k_1,\cdots,s^k_h,a^k_h,b^k_h\}$. Then we have for all $k\leq t-1$,
\begin{align*}
\E[W_{t,k}(h,g,\mu)|\Fil_{k,h}]=[(g_h-Q^{\mu,\nu^t}_h)(s_h^k,a_h^k,b_h^k)]^2,
\end{align*}
and 
\begin{align*}
\Var[W_{t,k}(h,g,\mu)|\Fil_{k,h}]\leq 4\VB^2\E[W_{t,k}(h,g,\mu)|\Fil_{k,h}].
\end{align*}
By Freedman's inequality, with probability at least $1-\delta/4$, we have
\begin{align*}
&\bigg|\sum_{k=1}^{t-1}W_{t,k}(h,g,\mu)-\sum_{k=1}^{t-1}[(g_h-Q^{\mu,\nu^t}_h)(s_h^k,a_h^k,b_h^k)]^2\bigg|\notag\\
&\leq\ \BO\Bigg(\VB\sqrt{\log\frac{1}{\delta}\cdot\sum_{k=1}^{t-1}[(g_h-Q^{\mu,\nu^t}_h)(s_h^k,a_h^k,b_h^k)]^2}+\VB^2\log\frac{1}{\delta}\Bigg).
\end{align*}

By taking union bound over $\Pi\times[K]\times[H]\times\VC$ and the non-negativity of $\sum_{k=1}^{t-1}[(g_h-Q^{\mu,\nu^t}_h)(s_h^k,a_h^k,b_h^k)]^2$, we have with probability at least $1-\delta/4$, for all $(\mu,k,h,g)\in\Pi\times[K]\times[H]\times\VC$,
\begin{equation*}
-\sum_{k=1}^{t-1}W_{t,k}(h,g,\mu)\leq\BO(\VB^2\iota),
\end{equation*}
where $\iota=\log({HK|\VC||\Pi|}/{\delta})$. This implies for all $(\mu,t,h,g)\in\Pi\times[K]\times[H]\times\GC$,
\begin{align*}
&\sum_{k=1}^{t-1}(Q^{\mu,\nu^t}_{h}(s_h^k,a_h^k,b_h^k)-r_h^k-Q^{\mu,\nu^t}_{h+1}(s^k_{h+1},\mu,\nu^t))^2\notag\\
&\qquad\leq\sum_{k=1}^{t-1}(g_h(s^k_h,a^k_h,b_h^k)-r_h^k-Q^{\mu,\nu^t}_{h+1}(s^k_{h+1},\mu,\nu^t))^2+\BO(\VB^2\iota+\VB t\rho).
\end{align*}
Choose $\rho=\VB/K$ and we know that with probability at least $1-\delta$ for all $\mu\in\Pi$ and $t\in[K]$, $Q^{\mu,\nu^t}\in\B_{\D_{1:t-1}}(\mu,\nu^t)$. This concludes our proof.

\subsection{Proof of Lemma~\ref{lem:bounded error}}
\label{proof lemma bounded error}
Let $\ZC$ be a $\rho$-cover of $\FC$ with respect to $\Vert\cdot\Vert_{\infty}$. Consider an arbitrary fixed tuple $(\mu,t,h,f)\in\Pi\times[K]\times[H]\times\FC$. Let 
\begin{align*}
	X_{t,k}(h,f,\mu):=&(f_h(s^k_h,a^k_h,b_h^k)-r_h^k-f_{h+1}(s^k_{h+1},\mu,\nu^t))^2\notag\\
	&-((\T^{\mu,\nu^t}_{h}f_{h+1})(s_h^k,a_h^k,b_h^k)-r_h^k-f_{h+1}(s^k_{h+1},\mu,\nu^t))^2,
\end{align*}
and $\Fil_{k,h}$ be the filtration induced by $\{\nu^1,\cdots,\nu^K\}\cup\{s^i_1,a^i_1,b^i_1,r^i_1,\cdots,s^i_{H+1}\}_{i=1}^{k-1}\cup\break\{s^k_1,a^k_1,b^k_1,r^k_1,\cdots,s^k_h,a^k_h,b^k_h\}$. Then we have for all $k\leq t-1$,
\begin{align*}
	\E[X_{t,k}(h,f,\mu)|\Fil_{k,h}]=[(f_h-\T^{\mu,\nu^t}_{h}f_{h+1})(s_h^k,a_h^k,b_h^k)]^2,
\end{align*}
and 
\begin{align*}
	\Var[X_{t,k}(h,f,\mu)|\Fil_{k,h}]\leq 4\VB^2\E[X_{t,k}(h,f,\mu)|\Fil_{k,h}].
\end{align*}
By Freedman's inequality, with probability at least $1-\delta$, 
\begin{align*}
	&\bigg|\sum_{k=1}^{t-1}X_{t,k}(h,f,\mu)-\sum_{k=1}^{t-1}[(f_h-\T^{\mu,\nu^t}_{h}f_{h+1})(s_h^k,a_h^k,b_h^k)]^2\bigg|\notag\\
	&\qquad\leq\ \BO\Bigg(\VB\sqrt{\log\frac{1}{\delta}\cdot\sum_{k=1}^{t-1}[(f_h-\T^{\mu,\nu^t}_{h}f_{h+1})(s_h^k,a_h^k,b_h^k)]^2}+\VB^2\log\frac{1}{\delta}\Bigg).
\end{align*}

By taking union bound over $\Pi\times[K]\times[H]\times\ZC$, we have with probability at least $1-\delta$, for all $(\mu,t,h,f)\in\Pi\times[K]\times[H]\times\ZC$,
\begin{align}	
&\bigg|\sum_{k=1}^{t-1}X_{t,k}(h,f,\mu)-\sum_{k=1}^{t-1}[(f_h-\T^{\mu,\nu^t}_{h}f_{h+1})(s_h^k,a_h^k,b_h^k)]^2\bigg|\notag\\
&\qquad\leq\ \BO\Bigg(\VB\sqrt{\iota\cdot\sum_{k=1}^{t-1}[(f_h-\T^{\mu,\nu^t}_{h}f_{h+1})(s_h^k,a_h^k,b_h^k)]^2}+\VB^2\iota\Bigg).\label{eq:conc-1}
\end{align}
where $\iota=\log({HK|\ZC||\Pi|}/{\delta})$. 

Conditioned on the above event being true, we consider an arbitrary pair $(h,t,\mu)\in[H]\times[K]\times\Pi$. By the definition of $\B_{\D_{1:t-1}}(\mu,\nu^t)$ and Assumption~\ref{ass:realize}, we have:
\begin{align*}
\sum_{k=1}^{t-1}X_{t,k}(h,f^{t,\mu},\mu)=&\sum_{k=1}^{t-1}(f_h(s^k_h,a^k_h,b_h^k)-r_h^k-f_{h+1}(s^k_{h+1},\mu,\nu^t))^2\notag\\
&-((\T^{\mu,\nu^t}_{h}f_{h+1})(s_h^k,a_h^k,b_h^k)-r_h^k-f_{h+1}(s^k_{h+1},\mu,\nu^t))^2\\
\leq&\sum_{k=1}^{t-1}(f_h(s^k_h,a^k_h,b_h^k)-r_h^k-f_{h+1}(s^k_{h+1},\mu,\nu^t))^2\notag\\
&-\inf_{g\in\GC}(g_h(s_h^k,a_h^k,b_h^k)-r_h^k-f_{h+1}(s^k_{h+1},\mu,\nu^t))^2\\
\leq&\beta.
\end{align*}
Let $l^{t,\mu}=\arg\min_{l\in\ZC}\max_{h\in[H]}\Vert f^{t,\mu}_h-l^{t,\mu}_h\Vert_{\infty}$. By the definition of $\ZC$, we have
\begin{align}
	\label{eq:bounded-1}
\sum_{k=1}^{t-1}X_{t,k}(h,l^{t,\mu},\mu)\leq \BO(\VB t\rho+\beta).
\end{align}
By (\ref{eq:conc-1}), we know:
\begin{align}	
	&\bigg|\sum_{k=1}^{t-1}X_{t,k}(h,l^{t,\mu},\mu)-\sum_{k=1}^{t-1}[(l^{t,\mu}_h-\T^{\mu,\nu^t}_{h}l^{t,\mu}_{h+1})(s_h^k,a_h^k,b_h^k)]^2\bigg|\notag\\
	&\qquad\leq\ \BO\Bigg(\VB\sqrt{\iota\cdot\sum_{k=1}^{t-1}[(l^{t,\mu}_h-\T^{\mu,\nu^t}_{h}l^{t,\mu}_{h+1})(s_h^k,a_h^k,b_h^k)]^2}+\VB^2\iota\Bigg).\label{eq:bounded-2}
\end{align}
Combining (\ref{eq:bounded-1}) and (\ref{eq:bounded-2}), we obtain
\begin{align*}
\sum_{k=1}^{t-1}[(l^{t,\mu}_h-\T^{\mu,\nu^t}_{h}l^{t,\mu}_{h+1})(s_h^k,a_h^k,b_h^k)]^2\leq\BO(\VB^2\iota+\VB t\rho+\beta).
\end{align*}
This implies that
\begin{align*}
	\sum_{k=1}^{t-1}[(f^{t,\mu}_h-\T^{\mu,\nu^t}_{h}f^{t,\mu}_{h+1})(s_h^k,a_h^k,b_h^k)]^2\leq\BO(\VB^2\iota+\VB t\rho+\beta).
\end{align*}
Choose $\rho=\VB/K$ and we can obtain (b). For (a), simply let $\Fil_{k,h}$ be the filtration induced by $\{\nu^1,\cdots,\nu^K\}\cup\{\mu^i,s^i_1,a^i_1,b^i_1,r^i_1,\cdots,s^i_{H+1}\}_{i=1}^{k-1}\cup\mu^k$ and repeat the above arguments, which concludes our proof. 

\subsection{Proof of Lemma~\ref{lem:performance}}
\label{proof lemma performance}
First notice that $\OVT(\mu^t)=f^{t,\mu^t}_1(s_1,\mu^t,\nu^t)$. Therefore, we have
\begin{align*}
&\OVT(\mu^t)-V^{\pi^t}_1(s_1)=\E_{a_1\sim\mu^t(\cdot|s_1),b_1\sim\nu^t(\cdot|s_1)}[f^{t,\mu^t}_1(s_1,a_1,b_1)-Q^{\pi^t}_1(s_1,a_1,b_1)]\\
&\qquad=\E_{a_1\sim\mu^t(\cdot|s_1),b_1\sim\nu^t(\cdot|s_1)}\big[\E_{s_2\sim P_1(\cdot|s_1,a_1,b_1)}[f^{t,\mu^t}_2(s_2,\mu^t,\nu^t)]-\E_{s_2\sim P_1(\cdot|s_1,a_1,b_1)}[V^{\pi^t}_2(s_2)]\big]\\
&\qquad\quad+\E_{a_1\sim\mu^t(\cdot|s_1),b_1\sim\nu^t(\cdot|s_1)}[(f^{t,\mu^t}_1-\T^{\mu^t,\nu^t}_1f_{2}^{t,\mu^t})(s_1,a_1,b_1)]\\
&\qquad=\E_{s_2\sim\pi^t}[f^{t,\mu^t}_2(s_2,\mu^t,\nu^t)-V^{\pi^t}_2(s_2)]+\E_{\pi^t}[(f^{t,\mu^t}_1-\T^{\mu^t,\nu^t}_1f_{2}^{t,\mu^t})(s_1,a_1,b_1)].
\end{align*}
Repeat the above procedures and we can obtain Lemma~\ref{lem:performance}. This concludes our proof.
\section{Proof of Corollary~\ref{cor:selfplay}}
\label{proof corollary selfplay}
From Theorem~\ref{thm:adaptive}, we have with probability at least $1-{\delta}$, for all $i\in[n]$
\begin{align*}
\max_{\mu_i\in\Pi_i}\frac{1}{K}\sum_{t=1}^KV_{1,i}^{\mu_i\times\mu_{-i}^t}(s_1)\leq\frac{1}{K}\sum_{t=1}^KV_{1,i}^{\mu_i^t\times\mu_{-i}^t}(s_1)+{\eps}.
\end{align*}
By the definition of $\widehat{\pi}$, this is equivalent to
\begin{align*}
\max_{\mu_i\in\Pi_i}V_{1,i}^{\mu_i\times\widehat{\mu}_{-i}}(s_1)\leq V_{1,i}^{\widehat{\pi}}(s_1)+{\eps},
\end{align*}
where $\widehat{\mu}_{-i}$ is uniformly sampled from $\{\mu^t_{-i}\}_{t=1}^K$ and thus is the marginal distribution of $\widehat{\pi}$ over the agents other than $i$. Therefore, by the definition of CCE in (\ref{eq:def cce}), $\widehat{\pi}$ is $\eps$-approximate CCE with probability at least $1-\delta$, which concludes our proof.
\section{Proof of Theorem~\ref{thm:cmdp}}
\label{proof:thm cmdp}
In this section we present the proof for Theorem~\ref{thm:cmdp}. Our proof mainly consists of four steps:
\begin{itemize}
	\item Prove $\OVTR(\mu)$ and $\OVTG(\mu)$ are optimistic estimations of $ \VR^{\mu}(s_1)$ and $\VG^{\mu}(s_1)$ for all $t\in[K]$ and $\mu\in\Pi$.
	\item Bound the total estimation error $\sum_{t=1}^K\OVTR(\mu^t)-\VR^{\mu^t}(s_1)$ and $\sum_{t=1}^K\OVTG(\mu^t)-\VG^{\mu^t}(s_1)$.
	\item Bound the regret by decomposing it into estimation error and online learning error induced by Hedge.
	\item Bound the constraint violation by strong duality.
\end{itemize}
\paragraph{Step 1: Prove optimism.}
First we can show that the constructed set $\B_{\D^r_{1:t-1}}(\mu)$ ($\B_{\D^g_{1:t-1}}(\mu)$) is not vacuous in the sense that the true action-value function $\QR^{\mu}$ ($\QG^{\mu}$) belongs to it with high probability:
\begin{lemma}
	\label{lem:optimism single}
	With probability at least $1-{\delta}/{4}$, we have for all $t\in[K]$ and $\mu\in\Pi$,
	\begin{align*}
		\QR^{\mu}\in\B_{\D^r_{1:t-1}}(\mu), \QG^{\mu}\in\B_{\D^g_{1:t-1}}(\mu).
	\end{align*}	
\end{lemma}
\begin{proof}
The proof is almost the same as Lemma~\ref{lem:optimism} and thus is omitted here.
\end{proof}
Then since $\OVTR(\mu)=\max_{f\in\B_{\D^g_{1:t-1}}(\mu)}f(s_1,\mu)$, we know for all $t\in[K]$ and $\mu\in\Pi$,
\begin{align*}
\OVTR(\mu)\geq \QR^{\mu}(s_1,\mu)=\VR^{\mu}(s_1).
\end{align*}
Similarly, we know $\OVTG(\mu)\geq\VG^{\mu}(s_1)$.

\paragraph{Step 2: Bound estimation error.} 
Next we need to show the estimation error $\sum_{t=1}^K\OVTR(\mu^t)-\VR^{\mu^t}(s_1)$ and $\sum_{t=1}^K\OVTG(\mu^t)-\VG^{\mu^t}(s_1)$ are small. Let $f^{t,\mu,r}=\arg\max_{f\in\B_{\D^r_{1:t-1}}(\mu)}f(s_1,\mu)$ and $f^{t,\mu,g}=\arg\max_{f\in\B_{\D^g_{1:t-1}}(\mu)}f(s_1,\mu)$. Then we have
\begin{lemma}
	\label{lem:bounded error single}
	With probability at least $1-{\delta}/{4}$, we have for all $t\in[K]$, $h\in[H]$ and $\mu\in\Pi$,
	\begin{align*}
	(a)\quad&\sum_{k=1}^{t-1}\E_{\mu^k}\bigg[\Big(f^{t,\mu,r}_h(s_h,a_h)-(\T^{\mu,r}_h f^{t,\mu,r}_{h+1})(s_h,a_h)\Big)^2\bigg]\leq\BO(\beta_r),\\
	&\sum_{k=1}^{t-1}\E_{\mu^k}\bigg[\Big(f^{t,\mu,g}_h(s_h,a_h)-(\T^{\mu,g}_h f^{t,\mu,g}_{h+1})(s_h,a_h)\Big)^2\bigg]\leq\BO(\beta_g),\\
	(b)\quad&\sum_{k=1}^{t-1}\Big(f^{t,\mu,r}_h(s_h^k,a_h^k)-(\T^{\mu,r}_h f^{t,\mu,r}_{h+1})(s_h^k,a_h^k)\Big)^2\leq\BO(\beta_r),\\
	&\sum_{k=1}^{t-1}\Big(f^{t,\mu,g}_h(s_h^k,a_h^k)-(\T^{\mu,g}_h f^{t,\mu,g}_{h+1})(s_h^k,a_h^k)\Big)^2\leq\BO(\beta_g).
	\end{align*}	
\end{lemma}
\begin{proof}
The proof is almost the same as Lemma~\ref{lem:bounded error} and thus is omitted here.
\end{proof}

Besides, using performance difference lemma we can easily bridge $\OVTR(\mu^t)-\VR^{\mu^t}(s_1)$ and $\OVTG(\mu^t)-\VG^{\mu^t}(s_1)$ with Bellman residuals, whose proof is also omitted:
\begin{lemma}
	\label{lem:performance single}
	For any $t\in[K]$, we have
	\begin{align*}
		&\OVTR(\mu^t)-\VR^{\mu^t}(s_1)=\sum_{h=1}^{H}\E_{\mu^t}[(f^{t,\mu^t,r}_h-\T^{\mu^t,r}f^{t,\mu^t,r}_{h+1})(s_h,a_h)],\\
		&\OVTG(\mu^t)-\VG^{\mu^t}(s_1)=\sum_{h=1}^{H}\E_{\mu^t}[(f^{t,\mu^t,g}_h-\T^{\mu^t,g}f^{t,\mu^t,g}_{h+1})(s_h,a_h)].
	\end{align*}
\end{lemma}

Therefore, from Lemma~\ref{lem:performance single} we can obtain for any $t\in[K]$,
\begin{align*}
\OVTR(\mu^t)-\VR^{\mu^t}(s_1)=\sum_{h=1}^{H}\E_{\mu^t}[(f^{t,\mu^t,r}_h-\T^{\mu^t,r}f^{t,\mu^t,r}_{h+1})(s_h,a_h)],
\end{align*}
which implies
\begin{align}
\label{eq:regret-bellman cmdp}
\sum_{t=1}^K\OVTR(\mu^t)-\VR^{\mu^t}(s_1)=\sum_{h=1}^{H}\sum_{t=1}^K\E_{\mu^t}[(f^{t,\mu^t,r}_h-\T_h^{\mu^t,r}f^{t,\mu^t,r}_{h+1})(s_h,a_h)].
\end{align}

Similar to Section~\ref{sec:proof sketch}, from Lemma~\ref{lem:BEE}, conditioning on the event in Lemma~\ref{lem:bounded error single} holds true, we have with probability at least $1-\delta/4$
\begin{align*}
&\sum_{t=1}^K\E_{\mu^t}[(f^{t,\mu^t,r}_h-\T^{\mu^t,r}_hf^{t,\mu^t,r}_{h+1})(s_h,a_h)]\notag\\
&\qquad\leq\BO\bigg(\sqrt{H^2K\DBE\Big(\mathcal{F},\sqrt{{1}/{K}},\Pi,r\Big)\log(\N_{\FCR\cup\GCR}(H/K)KH|\Pi|/\delta)}\bigg).
\end{align*}

Substitute the above bounds into (\ref{eq:regret-bellman cmdp}) and we have:
\begin{align}
&\sum_{t=1}^K\OVTR(\mu^t)-\VR^{\mu^t}(s_1)\notag\\
&\qquad\leq\BO\bigg(H^2\sqrt{K\DBE\Big(\FCR,\sqrt{{1}/{K}},\Pi,r\Big)\log(\N_{\FCR\cup\GCR}(H/K)KH|\Pi|/\delta)}\bigg). \label{eq:estimation r}
\end{align}
Similarly, we have
\begin{align}
	&\sum_{t=1}^K\OVTG(\mu^t)-\VG^{\mu^t}(s_1)\notag\\
	&\qquad\leq\BO\bigg(H^2\sqrt{K\DBE\Big(\FCG,\sqrt{{1}/{K}},\Pi,g\Big)\log(\N_{\FCG\cup\GCG}(H/K)KH|\Pi|/\delta)}\bigg).\label{eq:estimation g}
\end{align}

\paragraph{Step 3: Bound the regret.} 
Now we can bound the regret. We first decompose the fictitious total regret $\sum_{t=1}^{K}(\VR^{\OP}(s_1)+Y_t\OVTG(\OP))-\sum_{t=1}^{K}(\VR^{\mu^t}(s_1)+Y_t\OVTG(\mu^t))$ to the following terms:
\begin{align*}
&\sum_{t=1}^{K}(\VR^{\OP}(s_1)+Y_t\OVTG(\OP))-\sum_{t=1}^{K}(\VR^{\mu^t}(s_1)+Y_t\OVTG(\mu^t))\notag\\
&\qquad=\bigg(\underbrace{\sum_{t=1}^{K}\VR^{\OP}(s_1)-\sum_{t=1}^{K}\OVTR(\OP)}_{(1)}\bigg)\\
&\qquad\quad+\bigg(\underbrace{\sum_{t=1}^{K}(\OVTR(\OP)+Y_t\OVTG(\OP))-\sum_{t=1}^{K}\langle\OVTR+Y_t\OVTG,p^t\rangle}_{(2)}\bigg)\notag\\
&\qquad\quad+\bigg(\underbrace{\sum_{t=1}^{K}\langle\OVTR+Y_t\OVTG,p^t\rangle-\sum_{t=1}^{K}(\OVTR(\mu^t)+Y_t\OVTG(\mu^t))}_{(3)}\bigg)\notag\\
&\qquad\quad+\bigg(\underbrace{\sum_{t=1}^{K}\OVTR(\mu^t)-\sum_{t=1}^{K}\VR^{\mu^t}(s_1)}_{(4)}\bigg).
\end{align*}

From Lemma~\ref{lem:optimism single}, we know $(1)\leq0$. Since $p^t$ is updated using Hedge with loss function $\OVT$, we have $(2)\leq H(1+\rchi)\sqrt{K\log|\Pi|}$. $(3)$ is a martingale difference sequence, which implies $(3)\leq\BO\big(H(1+\rchi)\sqrt{K\log(1/\delta)}\big)$ with probability at least $1-\delta/4$. Finally, Step 2 has bounded term $(4)$ in (\ref{eq:estimation r}), which implies
\begin{align}
&\sum_{t=1}^{K}(\VR^{\OP}(s_1)+Y_t\OVTG(\OP))-\sum_{t=1}^{K}(\VR^{\mu^t}(s_1)+Y_t\OVTG(\mu^t))\notag\\
&\qquad\leq\BO\bigg(\bigg(H^2+\frac{H^2}{\sla}\bigg)\sqrt{K\dbeer\log\left(\ncovr|\Pi|/\delta\right)}\bigg)\label{eq:regrets-1}.
\end{align}

Now we only need to bound $-\sum_{t=1}^K Y_t(\OVTG(\muc)-\OVTG(\mu^t))$ if we want to bound the regret $\sum_{t=1}^K(\VR^{\muc}(s_1)-\VR^{\mu^t}(s_1))$. In fact, updating the dual variable $Y^t$ with projected gradient descent guarantees us the following lemma:
\begin{lemma}
\label{lem:dual 1}
Suppose the events in Lemma~\ref{lem:optimism single} hold true, we have
\begin{align*}
-\sum_{t=1}^K Y_t(\OVTG(\muc)-\OVTG(\mu^t))\leq\frac{\alpha H^2K}{2}=\frac{H^2\sqrt{K}}{2}.
\end{align*}
\end{lemma}
\begin{proof}
See Appendix~\ref{proof lemma dual 1}.
\end{proof}
Substituting Lemma~\ref{lem:dual 1} into (\ref{eq:regrets-1}), we can obtain the bound on $\reg(K)$:
\begin{align*}
&\sum_{t=1}^K(\VR^{\muc}(s_1)-\VR^{\mu^t}(s_1))\leq\BO\bigg(\bigg(H^2+\frac{H^2}{\sla}\bigg)\sqrt{K\dbeer\log\left(\ncovr|\Pi|/\delta\right)}\bigg).
\end{align*}

\paragraph{Step 4: Constraint Violation Analysis.}
Next we need to bound the constraint violation. First notice that $\sum_{t=1}^K Y_t(b-\OVTG(\mu^t))$ is indeed not far from $\sum_{t=1}^K Y(b-\OVTG(\mu^t))$ for any $Y\in[0,\rchi]$, as shown in the following lemma whose proof is deferred to Appendix~\ref{proof lemma dual 2}:
\begin{lemma}
	\label{lem:dual 2}
	For any $Y\in[0,\rchi]$, we have
	\begin{align*}
		\sum_{t=1}^K (Y-Y_t)(b-\OVTG(\mu^t))\leq\frac{(H^2+\rchi^2)\sqrt{K}}{2}.
	\end{align*}
\end{lemma}

Substituting Lemma~\ref{lem:dual 2} into (\ref{eq:regrets-1}) and notice that $b\leq\VG^{\muc}(s_1)\leq\OVTG(\muc)$, we have for any $Y\in[0,\rchi]$,
\begin{align*}
&\sum_{t=1}^K(\VR^{\muc}(s_1)-\VR^{\mu^t}(s_1))+Y\sum_{t=1}^K(b-\OVTG(\mu^t))\notag\\
&\qquad\leq\BO\bigg(\bigg(H^2+\frac{H^2}{\sla^2}\bigg)\sqrt{K\dbeer\log\left(\ncovr|\Pi|/\delta\right)}\bigg).
\end{align*}
Combining the above inequality with (\ref{eq:estimation g}), we have
\begin{align*}
\sum_{t=1}^K(\VR^{\muc}(s_1)-\VR^{\mu^t}(s_1))+Y\sum_{t=1}^K(b-\VG^{\mu^t}(s_1))\leq\BO\bigg(\bigg(\frac{H^2}{\sla^2}+\frac{H^3}{\sla}\bigg)\sqrt{K\ebe}\bigg),
\end{align*}
where
\begin{align*} \ebe=\max\bigg\{&\DBE\Big(\FCR,\sqrt{{1}/{K}},\Pi,r\Big)\log(\N_{\FCR\cup\GCR}(H/K)KH|\Pi|/\delta),\notag\\
&\DBE\Big(\FCG,\sqrt{{1}/{K}},\Pi,g\Big)\log(\N_{\FCG\cup\GCG}(H/K)KH|\Pi|/\delta)\bigg\}.
\end{align*}


Choose $Y$ as
\begin{align*}
Y=
\begin{cases}
0 & \text{if } \sum_{t=1}^K(b-\VG^{\mu^t}(s_1))<0,\\
\rchi & \text{otherwise}.
\end{cases}	
\end{align*}	
then we can bound the summation of regret and constraint violation as follows:
\begin{align}
&\bigg(\VR^{\muc}(s_1)-\frac{1}{K}\sum_{t=1}^K\VR^{\mu^t}(s_1)\bigg)+\rchi\bigg[b-\frac{1}{K}\sum_{t=1}^K\VG^{\mu^t}(s_1)\bigg]_{+}\notag\\
&\qquad\leq\BO\bigg(\bigg(\frac{H^2}{\sla^2}+\frac{H^3}{\sla}\bigg)\sqrt{{\ebe}/{K}}\bigg).\label{eq:cmdp4}
\end{align}

Further, when Assumption~\ref{ass:strong duality} and Assumption~\ref{ass:slater} hold, we have the following lemma showing that an upper bound on $(\VR^{\muc}(s_1)-\frac{1}{K}\sum_{t=1}^K\VR^{\mu^t}(s_1))+\rchi[b-\frac{1}{K}\sum_{t=1}^K\VG^{\mu^t}(s_1)]_{+}$ implies an upper bound on $[b-\frac{1}{K}\sum_{t=1}^K\VG^{\mu^t}(s_1)]_{+}$:
\begin{lemma}
\label{lem:dual 3}	
Suppose Assumption~\ref{ass:strong duality} and Assumption~\ref{ass:slater} hold and $2Y^*\leq C^*$. If $\{\mu^t\}_{t=1}^K\subseteq\Pi$ satisfies
\begin{align*}
\bigg(\VR^{\muc}(s_1)-\frac{1}{K}\sum_{t=1}^K\VR^{\mu^t}(s_1)\bigg)+C^*\bigg[b-\frac{1}{K}\sum_{t=1}^K\VR^{\mu^t}(s_1)\bigg]_{+}\leq\delta,
\end{align*}
Then
\begin{align*}
\bigg[b-\frac{1}{K}\sum_{t=1}^K\VR^{\mu^t}(s_1)\bigg]_{+}\leq\frac{2\delta}{C^*}.
\end{align*}
\end{lemma}

See Appendix~\ref{proof lemma dual 3} for the proof. Combining Lemma~\ref{lem:dual 3}, Lemma~\ref{lem:duality} and (\ref{eq:cmdp4}), we have
\begin{align*}
\bigg[\sum_{t=1}^K(b-\VG^{\mu^t}(s_1))\bigg]_{+}\leq\BO\bigg(\bigg(H^2+\frac{H}{\sla}\bigg)\sqrt{K\ebe}\bigg).
\end{align*}
This concludes our proof.

%

\subsection{Proof of Lemma~\ref{lem:duality}}
\label{proof lemma duality}
Notice that $D(Y^*)=\VR^{\muc}(s_1)$, which suggests:
\begin{align*}
\VR^{\muc}(s_1)&=D(Y^*)\geq\LO(\tmu,Y^*)\\
&=\VR^{\tmu}(s_1)+Y^*(\VG^{\tmu}(s_1)-b)\geq\VR^{\tmu}(s_1)+Y^*\sla.
\end{align*}	
This implies that
\begin{align*}
Y^*\leq\frac{\VR^{\muc}(s_1)-\VR^{\tmu}(s_1)}{\sla}\leq\frac{H}{\sla},
\end{align*}
which concludes our proof.

\subsection{Proof of Lemma~\ref{lem:dual 1}}
\label{proof lemma dual 1}
Notice that we have:
\begin{align*}
0&\leq Y_{K+1}^2=\sum_{t=1}^K\left(Y_{t+1}^2-Y_t^2\right)\\
&=\sum_{t=1}^K\left(\left(\proj(Y_t+\alpha(b-\OVTG(\mu^t)))\right)^2-Y_t^2\right)\\
&\leq\sum_{t=1}^K\left((Y_t+\alpha(b-\OVTG(\mu^t)))^2-Y_t^2\right)\\
&=\sum_{t=1}^K 2\alpha Y_t(b-\OVTG(\mu^t))+\sum_{t=1}^K\alpha^2(b-\OVTG(\mu^t))^2\\
&\leq\sum_{t=1}^K 2\alpha Y_t(\OVTG(\muc)-\OVTG(\mu^t))+\alpha^2KH^2,
\end{align*}
where the last step is due to optimism and $\VG^{\muc}(s_1)\geq b$. This implies that
\begin{align*}
-\sum_{t=1}^K Y_t(\OVTG(\muc)-\OVTG(\mu^t))\leq\frac{\alpha H^2K}{2}=\frac{H^2\sqrt{K}}{2}.
\end{align*}
This concludes our proof.

\subsection{Proof of Lemma~\ref{lem:dual 2}}
\label{proof lemma dual 2}
Notice that we have for any $t\in[K]$ and $Y\in[0,\rchi]$:
\begin{align*}
|Y_{t+1}-Y|^2&\leq|Y_t+\alpha(b-\OVTG(\mu^t))-Y|^2\\
&=(Y_t-Y)^2+2\alpha(b-\OVTG(\mu^t))(Y_t-Y)+\alpha^2H^2.
\end{align*}
Repeating the above expansion procedures, we have
\begin{align*}
0\leq|Y_{K+1}-Y|^2\leq(Y_1-Y)^2+2\alpha\sum_{t=1}^K(b-\OVTG(\mu^t))(Y_t-Y)+\alpha^2H^2K,
\end{align*}
which is equivalent to
\begin{align*}
\sum_{t=1}^K(b-\OVTG(\mu^t))(Y-Y_t)\leq\frac{1}{2\alpha}(Y_1-Y)^2+\frac{\alpha}{2}H^2K\leq\frac{(H^2+\rchi^2)\sqrt{K}}{2}.
\end{align*}
This concludes our proof.

\subsection{Proof of Lemma~\ref{lem:dual 3}}
\label{proof lemma dual 3}
First we extend $\Pi$ in a reasonable way to make the policy class more structured while not changing its optimal policy. Define the set of state-action visitation distributions induced by the policy $\Pi$ as follows:
\begin{align*}
\dpi_{\Pi}=\{(d^{\mu}_h(s,a))_{h\in[H],s\in\SC,a\in\A}\in(\Delta_{|\SC|\times|\A|})^H:\mu\in\Pi\}.
\end{align*}
Let $\cpi$ denote the convex hull of $\dpi_{\Pi}$, i.e., for any $d\in\cpi$, there exists $\{w_{\mu}\}_{\mu\in\Pi}\geq0$ such that for any $h\in[H],s\in\SC.a\in\A$, we have 
\begin{align*}
d_h(s,a)=\sum_{\mu\in\Pi}w_{\mu}d^{\mu}_h(s,a),\sum_{\mu\in\Pi}w_{\mu}=1.
\end{align*} 
As a special case, there exists $d'_h(s,a)\in\cpi$ such that for any $h\in[H],s\in\SC.a\in\A$,
\begin{align*}
d'_h(s,a)=\frac{1}{K}\sum_{t=1}^Kd^{\mu^t}_h(s,a).
\end{align*}

Notice that there exists a one-to-one mapping from state-action visitation distributions to policies \citep{puterman1994markov}. Let $\cpc$ denote the policy class that induces $\cpi$, and then there exists $\mu'$ such that $d'=d^{\mu'}$, which implies 
\begin{align*}
\VR^{\mu'}(s_1)=\frac{1}{K}\sum_{t=1}^K\VR^{\mu^t}(s_1),\VG^{\mu'}(s_1)=\frac{1}{K}\sum_{t=1}^K\VG^{\mu^t}(s_1).
\end{align*}
Therefore, the condition of this lemma says
\begin{align}
\label{eq:dual-4}
(\VR^{\muc}(s_1)-\VR^{\mu'}(s_1))+C^*[b-\VG^{\mu'}(s_1)]_{+}\leq\delta.
\end{align}

Next we show that $\muc$ is still the optimal policy in $\cpc$ when Assumption~\ref{ass:strong duality}, i.e., strong duality, holds. First notice that
\begin{align}
	\label{eq:dual-1}
\max_{\mu\in\cpc}\min_{Y\geq0}\LO(\mu,Y)\leq\min_{Y\geq0}\max_{\mu\in\cpc}\LO(\mu,Y)=\min_{Y\geq0}\max_{d\in\cpi}\LO(d,Y).
\end{align}
However, given $Y\geq0$,$\LO(d,Y)$ is linear in $d$, which means the maximum is always attained at the vertices of $\cpi$, i.e., $\dpi_{\Pi}$. Therefore we know
\begin{align*}
\max_{\mu\in\cpc}\LO(\mu,Y)=D(Y),
\end{align*}
which suggests
\begin{align}
\label{eq:dual-2}
\min_{Y\geq0}\max_{d\in\cpi}\LO(d,Y)=\min_{Y\geq0}\max_{d\in\dpi_{\Pi}}\LO(d,Y)=\min_{Y\geq0}\max_{\mu\in\Pi}\LO(\mu,Y).
\end{align}
By strong duality, we have
\begin{align}
\label{eq:dual-3}
\min_{Y\geq0}\max_{\mu\in\Pi}\LO(\mu,Y)=\max_{\mu\in\Pi}\min_{Y\geq0}\LO(\mu,Y)\leq\max_{\mu\in\cpc}\min_{Y\geq0}\LO(\mu,Y).
\end{align}
Combining (\ref{eq:dual-1}),(\ref{eq:dual-2}) and (\ref{eq:dual-3}), we know all the inequalities have to take equality, which implies
\begin{align*}
\muc=\arg\max_{\mu\in\cpc}\min_{Y\geq0}\LO(\mu,Y), Y^*=\arg\min_{Y\geq0}\max_{\mu\in\cpc}\LO(\mu,Y).
\end{align*}
Besides, strong duality also holds for $\max_{\mu\in\cpc}\min_{Y\geq0}\LO(\mu,Y)$.

Now let $v(\tau):=\max_{\mu\in\cpc}\{\VR^{\mu}(s_1)|\VG^{\mu}(s_1)\geq b+\tau\}$, then we have for any $\mu\in\cpc$, 
\begin{align*}
\LO(\mu,Y^*)&\leq\max_{\mu\in\cpc}\LO(\mu,Y^*)=D(Y^*)=\VR^{\muc}(s_1),
\end{align*}
where the third step comes from strong duality. Therefore, for any $\mu\in\cpc$ and $\tau\in\R$ which satisfies $\VG^{\mu}(s_1)\geq b+\tau$, we have
\begin{align*}
\VR^{\muc}(s_1)-\tau Y^*&\geq\LO(\mu,Y^*)-\tau Y^*\\
&=\VR^{\mu}(s_1)+Y^*(\VG^{\mu}(s_1)-b-\tau)\geq\VR^{\mu}(s_1).
\end{align*}
This implies that for any $\tau\in\R$, $\VR^{\muc}(s_1)-\tau Y^*\geq v(\tau)$. Pick $\tau=\ttau:=-[b-\VG^{\mu'}(s_1)]_+$, then we have
\begin{align*}
\VR^{\mu'}(s_1)-\VR^{\muc}(s_1)\leq -\ttau Y^*.
\end{align*}
On the other hand, (\ref{eq:dual-4}) is equivalent to
\begin{align*}
\VR^{\muc}(s_1)-\VR^{\mu'}(s_1)-C^*\ttau\leq\delta.
\end{align*}
Thus we have $(C^*-Y^*)|\ttau|\leq\delta$, which means that
\begin{align*}
[b-\VG^{\mu'}(s_1)]_+\leq\frac{\delta}{C^*-Y^*}\leq\frac{2\delta}{C^*}.
\end{align*}
Recall that $\VG^{\mu'}(s_1)=\frac{1}{K}\sum_{t=1}^K\VG^{\mu^t}(s_1)$, which concludes our proof.
\section{Proof of Theorem~\ref{thm:vmdp}}
\label{proof:thm vmdp}
In this section we present the proof for Theorem~\ref{thm:vmdp}. Our proof mainly consists of four steps:
\begin{itemize}
	\item Prove $\langle\PET(\mu),\bth_t\rangle$ is a pessimistic estimations of $ \langle \bv^{\mu}_1(s_1),\bth_t\rangle$ for all $t\in[K]$ and $\mu\in\Pi$.
	\item Bound the total estimation error $\Vert\frac{1}{K}\sum_{t=1}^K\PET(\mu^t)-\bv_1^{\mu^t}(s_1)\Vert$.
	\item Bound $\dist(\bv^{\widehat{\mu}}_1(s_1),\CS)$.
\end{itemize}
\paragraph{Step 1: Prove pessimism.}
First we can show that the true action-value function $\bq^{\mu}$ belongs to the constructed set $\B_{\D_{1:t-1}}(\mu)$ with high probability:
\begin{lemma}
	\label{lem:optimism vmdp}
	With probability at least $1-{\delta}/{4}$, we have for all $t\in[K]$ and $\mu\in\Pi$, $\bq^{\mu}\in\B_{\D_{1:t-1}}(\mu)$.
\end{lemma}
\begin{proof}
Repeat the arguments in the proof of Lemma~\ref{lem:optimism} for each dimension $j\in[d]$ and the lemma follows directly.
\end{proof}
Then since $\PET(\mu)=f_1(s_1,\mu)$ where $f=\arg\min_{f'\in\B_{\D_{1:t-1}}(\mu)}\langle f'_1(s_1,\mu),\bth_t\rangle$, we know for all $t\in[K]$ and $\mu\in\Pi$,
\begin{align*}
\langle\PET(\mu),\bth_t\rangle\leq \langle\bq^{\mu}_1(s_1,\mu),\bth_t\rangle=\langle\bv^{\mu}_1(s_1),\bth_t\rangle.
\end{align*}

\paragraph{Step 2: Bound estimation error.} 
Next we need to show the estimation error $\Vert\frac{1}{K}\sum_{t=1}^K\break\PET(\mu^t)-\bv_1^{\mu^t}(s_1)\Vert$ is small. Let $f^{t,\mu}=\arg\min_{f\in\B_{\D_{1:t-1}}(\mu)}\langle f_1(s_1,\mu),\bth_t\rangle$. Let $f^{t,\mu,j}$ denotes the $j$-the dimension of $f^{t,\mu}$. Then we have
\begin{lemma}
	\label{lem:bounded error vmdp}
	With probability at least $1-{\delta}/{4}$, we have for all $t\in[K]$, $h\in[H]$, $j\in[d]$ and $\mu\in\Pi$,
	\begin{align*}
	(a)\quad&\sum_{k=1}^{t-1}\E_{\mu^k}\bigg[\Big(f^{t,\mu,j}_h(s_h,a_h)-(\T^{\mu,j}_h f^{t,\mu,j}_{h+1})(s_h,a_h)\Big)^2\bigg]\leq\BO(\beta),\\
	(b)\quad&\sum_{k=1}^{t-1}\Big(f^{t,\mu,j}_h(s_h^k,a_h^k)-(\T^{\mu,j}_h f^{t,\mu,j}_{h+1})(s_h^k,a_h^k)\Big)^2\leq\BO(\beta),\\
	\end{align*}	
\end{lemma}
\begin{proof}
Repeat the arguments in the proof of Lemma~\ref{lem:bounded error} for each dimension $j\in[d]$ and the lemma follows directly.
\end{proof}

Besides, using performance difference lemma we have:
\begin{lemma}
	\label{lem:performance vmdp}
	For any $t\in[K]$ and $j\in[d]$, we have
	\begin{align*}
	\PETJ(\mu^t)-V^{\mu^t,j}_{1}(s_1)=\sum_{h=1}^{H}\E_{\mu^t}[(f^{t,\mu^t,j}_h-\T^{\mu^t,j}f^{t,\mu^t,j}_{h+1})(s_h,a_h)],
	\end{align*}
    where $\PETJ(\mu^t)$ is the $j$-th dimension of $\PET(\mu^t)$.
\end{lemma}

Therefore, from Lemma~\ref{lem:performance vmdp} we can obtain for any $t\in[K]$ and $j\in[d]$ 
\begin{align}
\label{eq:regret-bellman vmdp}
\sum_{t=1}^K\PETJ(\mu^t)-V^{\mu^t,j}_{1}(s_1)=\sum_{h=1}^{H}\sum_{t=1}^K\E_{\mu^t}[(f^{t,\mu^t,j}_h-\T_h^{\mu^t,j}f^{t,\mu^t,j}_{h+1})(s_h,a_h)].
\end{align}

Similar to Section~\ref{sec:proof sketch}, from Lemma~\ref{lem:BEE}, conditioning on the event in Lemma~\ref{lem:bounded error vmdp} holds true, with probability at least $1-\delta/4$, we have for any $j\in[d]$ and $h\in[H]$,
\begin{align*}
&\bigg|\sum_{t=1}^K\E_{\mu^t}\big[(f^{t,\mu^t,j}_h-\T^{\mu^t,j}_hf^{t,\mu^t,j}_{h+1})(s_h,a_h)\big]\bigg|\leq\BO\Big(\sqrt{H^2K\dbeev\log\left(\ncovv|\Pi|d/\delta\right)}\Big).
\end{align*}

Substitute the above bounds into (\ref{eq:regret-bellman vmdp}) and we have for any $j\in[d]$:
\begin{align*}
\bigg|\sum_{t=1}^K\PETJ(\mu^t)-V^{\mu^t,j}_{1}(s_1)\bigg|\leq\BO\big(H^2\sqrt{K\dbeev\log(\ncovv|\Pi|d/\delta)}\big),
\end{align*}
which implies if the event in Lemma~\ref{lem:bounded error vmdp} is true,
\begin{align*}
\bigg\Vert\frac{1}{K}\sum_{t=1}^K\PET(\mu^t)-\bv^{\mu^t}_1(s_1)\bigg\Vert\leq\BO\big(H^2\sqrt{d}\cdot\sqrt{{\dbeev\log(\ncovv|\Pi|d/\delta)}/{K}}\big).
\end{align*}

\paragraph{Step 3: Bound the distance.} 
Now we can bound the distance $\dist(\bv^{\widehat{\mu}}(s_1),\CS)$. First since $\widehat{\mu}$ is sampled uniformly from $\{\mu^t\}_{t=1}^K$, we know
\begin{align*}
\dist(\bv^{\widehat{\mu}}_1(s_1),\CS)=\dist\bigg(\frac{1}{K}\sum_{t=1}^K\bv^{\mu^t}_1(s_1),\CS\bigg).
\end{align*}
By Fenchel's duality, we know
\begin{align*}
&\dist\bigg(\frac{1}{K}\sum_{t=1}^K\bv^{\mu^t}_1(s_1),\CS\bigg)=\max_{\bth\in\BL(1)}\bigg[\bigg\langle \bth,\frac{1}{K}\sum_{t=1}^K\bv^{\mu^t}_1(s_1)\bigg\rangle-\max_{\bx\in\CS}\langle\bth,\bx\rangle\bigg]\\
&\qquad\leq\max_{\bth\in\BL(1)}\bigg[\bigg\langle \bth,\frac{1}{K}\sum_{t=1}^K\PET(\mu^t)\bigg\rangle-\max_{\bx\in\CS}\langle\bth,\bx\rangle\bigg]+\max_{\bth\in\BL(1)}\bigg\langle \bth,\frac{1}{K}\sum_{t=1}^K\bv^{\mu^t}_1(s_1)-\PET(\mu^t)\bigg\rangle,
\end{align*}
where the second step is due to $\max[f_1+f_2]\leq\max f_1+\max f_2$.

Notice by Cauchy-Schwartz inequality and Step 2, we have
\begin{align*}
&\max_{\bth\in\BL(1)}\bigg\langle \bth,\frac{1}{K}\sum_{t=1}^K\bv^{\mu^t}_1(s_1)-\PET(\mu^t)\bigg\rangle\leq\bigg\Vert\frac{1}{K}\sum_{t=1}^K\PET(\mu^t)-\bv^{\mu^t}_1(s_1)\bigg\Vert\\
&\qquad\leq\BO\big(H^2\sqrt{d}\cdot\sqrt{{\dbeev\log(\ncovv|\Pi|d/\delta}/{K}}\big).
\end{align*}

Now we only need to bound $\max_{\bth\in\BL(1)}\left[\langle \bth,\frac{1}{K}\sum_{t=1}^K\PET(\mu^t)\rangle-\max_{\bx\in\CS}\langle\bth,\bx\rangle\right]$. Recall that we update $\bth_t$ using online gradient descent. Using the conclusions from the online learning literature \citep{hazan2016introduction}, we know
\begin{align*}
&\max_{\bth\in\BL(1)}\bigg[\bigg\langle \bth,\frac{1}{K}\sum_{t=1}^K\PET(\mu^t)\bigg\rangle-\max_{\bx\in\CS}\langle\bth,\bx\rangle\bigg]\\
&\qquad\leq\frac{1}{K}\sum_{t=1}^K\Big(\langle\bth_t,\PET(\mu^t)\rangle-\max_{x\in\CS}\langle\bth_t,x\rangle\Big)+\BO({H\sqrt{d}}/{\sqrt{K}}).
\end{align*}

Further, notice that $p^t$ is updated via Hedge with loss function being $\langle\bth_t,\PET(\mu)\rangle$, similarly to the analysis in Section~\ref{sec:proof sketch}, we have with probability at least $1-\delta$,
\begin{align*}
\frac{1}{K}\sum_{t=1}^K\langle\bth_t,\PET(\mu^t)\rangle\leq\frac{1}{K}\sum_{t=1}^K\langle\bth_t,\PET(\muv)\rangle+\BO(H\sqrt{d}\cdot\sqrt{\log(|\Pi|/\delta) / K}),
\end{align*}
where $\muv=\arg\min_{\mu\in\Pi}\dist(\bv^{\mu}_1(s_1),\CS)$. Let $P(\bv^{\muv}_1(s_1))$ denote the projection of $\bv^{\muv}_1(s_1)$ onto $\CS$.

Conditioning on the event of Lemma~\ref{lem:optimism vmdp} holds, we have
\begin{align*}
	\sum_{t=1}^K\langle\bth_t,\PET(\muv)\rangle\leq\sum_{t=1}^K\langle\bth_t,\bv^{\muv}_1(s_1)\rangle.
\end{align*} 
Therefore we have
\begin{align*}
&\frac{1}{K}\sum_{t=1}^K\left(\langle\bth_t,\PET(\mu^t)\rangle-\max_{x\in\CS}\langle\bth_t,x\rangle\right)\\
&\qquad \leq \frac{1}{K}\sum_{t=1}^K\left(\langle\bth_t,\bv^{\muv}_1(s_1)\rangle-\max_{x\in\CS}\langle\bth_t,x\rangle\right)+\BO \bigl (H\sqrt{d}\cdot\sqrt{\log(|\Pi|/\delta) / K} \bigr)\\
 &\qquad \leq \frac{1}{K}\sum_{t=1}^K\left(\langle\bth_t,\bv^{\muv}_1(s_1)\rangle-\langle\bth_t,\pmu\rangle\right)+\BO \bigl (H\sqrt{d}\cdot\sqrt{\log(|\Pi|/\delta) / K}\bigr)\\
&\qquad \leq \left\Vert\bv^{\muv}_1(s_1)-\pmu\right\Vert+\BO \bigl (H\sqrt{d}\cdot\sqrt{\log(|\Pi|/\delta) / K}\bigr)\\
 &\qquad  = \min_{\mu\in\Pi}\dist(\bv^{\mu}_1(s_1),\CS)+\BO \bigl (H\sqrt{d}\cdot\sqrt{\log(|\Pi|/\delta) / K} \bigr),
\end{align*}
where the second step is due to $\pmu\in\CS$, the third step is from Cauchy-Schwartz inequality, and the last step is from the definition of $\muv$.

In conclusion, we have with probability at least $1-\delta$,
\begin{align*}
\dist(\bv^{\widehat{\mu}}_1(s_1),\CS)\leq\min_{\mu\in\Pi}\dist(\bv^{\mu}_1(s_1),\CS)+\BO\big(H^2\sqrt{d}\cdot\sqrt{{\dbeev\log\left(\ncovv|\Pi|d/\delta\right)}/{K}}\big).
\end{align*}
This concludes our proof.
\end{document}